\pdfoutput=1
\documentclass[11pt,a4paper]{article}
\usepackage[T1]{fontenc}
\usepackage[utf8]{inputenc}
\usepackage{lmodern,microtype}
\usepackage[margin=27mm]{geometry}
\usepackage{amsmath,amssymb,amsthm,mathtools}
\usepackage{enumitem,booktabs,float}
\usepackage[dvipsnames]{xcolor}
\usepackage{tikz}
\usetikzlibrary{arrows.meta,positioning}
\usepackage[numbers,sort&compress]{natbib}
\usepackage[colorlinks=true,linkcolor=MidnightBlue,citecolor=MidnightBlue,urlcolor=MidnightBlue]{hyperref}
\usepackage{url}
\hypersetup{pdftitle={Structure-Aware Message Passing on Petri Nets through Directed Walks and Incidence Diffusion}}
\newcommand{\R}{\mathbb{R}}
\newcommand{\N}{\mathbb{N}}
\newcommand{\ones}{\mathbf{1}}
\newcommand{\diag}{\operatorname{diag}}
\newcommand{\norm}[1]{\left\lVert#1\right\rVert}
\newtheorem{proposition}{Proposition}
\setlist{nosep,leftmargin=*}
\providecommand{\pmsd}[2]{#1$\pm$#2}

\title{Structure-Aware Message Passing on Petri Nets through Directed Walks and Incidence Diffusion}
\author{Geri Skenderi\\HASH\\g@hash.ai}
\date{October 21, 2026}

\begin{document}
\maketitle
\begin{abstract}
A common application of the process mining discipline is to connect observed event data with Petri nets, which are bipartite, directed graphical process models with rich formal underpinnings. Recent work learns vector representations of Petri nets to retrieve process models and relate them to event logs, but leaves implicit two key aspects hidden within the net's structure: directed cycles (process loops) and invariants recorded by the net's incidence matrix. We expose both of these to message-passing graph neural networks, through a positional encoding built from return probabilities of directed place-transition walks (PNRWPE), and through same-type messages gated by diffusion kernels of the incidence matrix (DA-PNGNN). Our constructions raise macro balanced accuracy in predicting the generation factors of the 96 Process Discovery Contest 2023 models from 0.751 to 0.796 with PNRWPE, which nearly doubles accuracy on process loops, and to 0.830 with DA-PNGNN, which, without retuning, also improves net-to-tree retrieval top-1 accuracy on the self-supervised ProcRosetta benchmark by 4.1 to 5.9 points.
\end{abstract}

\section{Introduction}
Petri nets are a commonly used formal model of discrete-event dynamical systems supporting the analysis of concurrent computation~\cite{murata}. This has made them popular in the process-mining literature, which constructs, checks, and simulates them routinely~\cite{vanderaalstProcessMining3602022}. Other process-mining tasks, such as retrieving similar models from a collection or relating a model to the event logs of the same process, can be approached by learning vector representations of the nets themselves. Existing work has learned such representations for retrieval and classification~\cite{colonna}, and a recent study~\cite{berti} has built a multimodal neural model that aligns the process net with the process tree that generated it~\cite{vanderaalstProcessMining3602022} and the associated event logs.

A Petri net is, stripped to its bare bones, a directed bipartite graph whose two node partitions (or types) have semantic roles, and whose incidence matrix carries an algebra of conservation that Petri net theory has studied for decades~\cite{murata}. Message passing is the de facto paradigm for learning on graphs~\cite{gilmer2020message}. It relies on repeated local aggregation of learned node features along the edges of the graph. It is therefore a natural tool for learning representations of a Petri net, and simple adaptations of message-passing networks have been shown to be beneficial~\cite{berti}. We argue that current modeling approaches are not making complete use of specific structural inductive biases that a Petri net structure can provide. Concrete examples are notions of where a node sits relative to directed cycles, and whether paths between node partitions that lead to potential token conservation exist. This raises the central question that motivates our work: can the structure that Petri net theory makes explicit improve learned representations of process models?

In this work, we answer this question by proposing two constructions that extract higher-order structural information from a Petri net for use in message passing. The first concerns positional encodings (PEs). Since graphs are structured domains, a sense of position lets a node encode its graph-based distance to other nodes, or its membership in substructures that matter for a downstream task~\cite{dwivedi2020generalization}. A positional encoding implements this as a deterministic vector describing a node's structural role, which can be fed as input to any neural model over the graph, and a popular instantiation is the random-walk positional encoding (RWPE)~\cite{dwivedi}. We extend this encoding to Petri nets, obtaining the Petri net random-walk positional encoding (PNRWPE), by building it on the return probabilities of an alternating place-transition walk, a construction similar in spirit to the one used by transformers for directed graphs~\cite{geisler2023transformers}. Since directed returns collapse on acyclic nets, we also derive a variant with uniform teleportation between opposite node types, which acts as a directed, bipartite PageRank~\cite{bianchini2005inside}. Both encodings are equivariant to node reordering and can be cheaply computed for a new net.

The second construction is a modification of the message-passing procedure itself. The incidence matrix of a net records how each transition consumes and produces tokens, which lets us define a diffusion operator whose steady state lies in the place- or transition-invariant subspaces of the net~\cite{murata}. Its powers yield pairwise kernels between nodes of the same type that suppress violations of linear token balance, and we use these kernels to learn same-type messages in addition to the ordinary ones. The resulting higher-order message passing produces representations that are informed about the state and structural conservation properties of the Petri net, an inductive bias unique to this graphical model. This construction is compatible with any message-passing architecture and can be used together with PNRWPE. We call the resulting architecture the diffusion-augmented Petri net graph neural network (DA-PNGNN).

We derive both constructions from first principles and evaluate them on two benchmarks from recent work: supervised prediction of the generation factors of the Process Discovery Contest (PDC) 2023 process models~\cite{colonna}, and the self-supervised multimodal objective of ProcRosetta~\cite{berti}. Both constructions provide improvements and the diffusion kernels do so on both benchmarks: on PDC 2023 they raise macro balanced accuracy from 0.751 to 0.832 when used together, while on ProcRosetta DA-PNGNN improves net-to-tree retrieval by up to 5.9 top-1 accuracy points. We run several control experiments that isolate the contribution of each component, study the effect of the receptive field during message passing, and also report the computational cost of the main architectural variants. The rest of the paper is organized as follows: Section~\ref{sec:prelim} provides some background and fixes notation, Sections~\ref{sec:rwpe} and~\ref{sec:da} derive the two constructions, Section~\ref{sec:experiments} reports the complete experimental protocol and results, and Section~\ref{sec:conclusion} discusses what the evidence supports and our final conclusions.

\section{Preliminaries}\label{sec:prelim}
This section provides a brief account of the two ingredients that our constructions build on. Section~\ref{sec:gnn} summarizes message passing and random-walk positional encodings on ordinary graphs, and Section~\ref{subsec:petri-structure} introduces the matrices that describe a Petri net. We denote matrices by upper-case letters ($A$, $\Phi$), vectors and scalars by lower-case, with the role of each symbol clear from the context. For a matrix $A$, $A_{ij}$ is its $(i,j)$-th entry, $(A^k)_{ij}$ the corresponding entry of its $k$-th power, $A^\top$ its transpose, and $\diag(A)$ the vector of its diagonal entries. Conversely, $\diag(v)$ is the diagonal matrix built from a vector $v$. The symbol $\ones$ denotes the all-ones vector or matrix and $I$ the identity matrix. Their shape is given as a subscript and is otherwise to be understood from the context. $\N_0$ denotes the nonnegative integers, $\R$ the reals, $\ker$ the null space of a matrix, and $\det$ its determinant. Finally, $\norm{\cdot}_2$ denotes the Euclidean norm for vectors and the spectral norm for matrices.

\subsection{Graph neural networks}\label{sec:gnn}
Graph neural networks (GNNs) are a subclass of neural networks that learn representations of graph data by locally aggregating information~\cite{scarselli2008graph,gilmer2020message}. The main goal of the architecture is to implement symmetries natural to graphs~\cite{bronstein2017geometric}. Consider, for simplicity, an unweighted and undirected graph $G$ with $n$ nodes, represented by a symmetric binary adjacency matrix $A \in \{0,1\}^{n \times n}$. This setting can be easily extended to deal with more general connectivity structures, as we will do in this work. By associating a node signal $X \in \mathbb{R}^{n \times m}$ with the graph, we can describe a simple GNN as a learnable convolution of the node signal with $A$ as the shift operator. A generalization of this concept can be obtained by considering the message-passing framework:
\begin{equation}\label{eq:msg-passing}
    x_i^{(\ell)} = \theta^{(\ell)} \left( x_i^{(\ell-1)}, \bigoplus_{j \in \mathcal{N}(i)} \, \phi^{(\ell)}\left(x_i^{(\ell-1)}, x_j^{(\ell-1)},e_{ji}\right) \right),
\end{equation}
where $x_i^{(\ell)}$ denotes the features of node $i$ at layer $\ell$, $e_{ji}$ the (optional) edge features of the connection between nodes $j$ and $i$, $\mathcal{N}(\cdot)$ the set of (1-hop) neighbor nodes, $\bigoplus$ a differentiable, permutation-invariant function (e.g., sum, mean), and $\phi, \theta$ denote differentiable and (optionally) nonlinear functions such as multilayer perceptrons. Note that $\phi$ is of crucial importance as it allows us to transform the node features before aggregation, thereby acting as a map that determines what the aggregate can record.

\paragraph{Positional encodings and random walks.} Positional encodings complement the node signal $X$ with information
about each node's relative position (or structural role) in the graph. We remind the reader that a canonical coordinate system cannot be fixed on a graph and therefore the naive notion of an absolute, global positional encoding cannot be employed. In a certain sense, a graph positional encoding can be thought of as attempting to provide information about the relative positioning of a node with respect to others. In this work, we focus on random-walk positional encodings~\cite{dwivedi}. For the undirected graph $G$, let $d_i = \sum_{j=1}^{n} A_{ij}$ denote the degree of node $i$. The random-walk transition matrix $Q \in [0,1]^{n \times n}$ is defined by $Q_{ij}=A_{ij}/d_i$ whenever $d_i>0$. Given a maximum walk length $K$, the RWPE of node $i$ is
\begin{equation}\label{eq:rwpe-undirected}
    r_i =
    \left[
        Q_{ii},
        (Q^2)_{ii},
        \ldots,
        (Q^K)_{ii}
    \right]
    \in [0,1]^K.
\end{equation}
The coordinate $(Q^k)_{ii}$ is the probability that a random walk starting at node $i$ returns to that node after $k$ steps.
The relative position of the node is therefore captured through the return patterns at different walk lengths: high values throughout the components of this position vector imply the node lives in a dense cluster within the graph. In this scenario, the positional encoding will be similar for nodes in that cluster. The encoding can be incorporated directly into features fed into a GNN, for example through concatenation with $x^{(0)}_i$ in Equation~\eqref{eq:msg-passing}. 

\subsection{Petri net structure}\label{subsec:petri-structure}
We consider finite Petri nets with place set $P$, transition set $T$, and nonnegative integer arc multiplicities $w$. Let $n_P=|P|$, $n_T=|T|$, and $n=n_P+n_T$. We assume $n_P,n_T>0$ and allow for disconnected components or isolated nodes. Inhibitor, timed, and colored-net semantics require additional structure and lie outside the scope of this study. The matrices $E^-,E^+\in\N_0^{n_P\times n_T}$, with entries $E^-_{ij}=w(p_i,t_j)$ and $E^+_{ij}=w(t_j,p_i)$, record token consumption ($P \to T$) and production ($T \to P$), respectively. In both cases we index using places as rows and transitions as columns. The Petri net's structure can be represented by the following directed adjacency matrix, written in block form as
\begin{equation}
 A=\begin{pmatrix}0&E^-\\(E^+)^\top&0\end{pmatrix}\in\N_0^{n\times n}.
 \label{eq:adj}
\end{equation}
Notice the bipartition: there are no $P \to P$ or $T \to T$ arcs and $P \cap T = \emptyset$. The \emph{incidence matrix} is defined as $ B=E^+-E^-\in\mathbb{Z}^{n_P\times n_T}$. It encodes the net change of tokens at each place when a transition fires. This object is fundamental in Petri net theory, as it captures the structural relationships between places and transitions~\cite{murata}. For a marking (token state) $m\in\N_0^{n_P}$, a transition $t_j$ is enabled when $m_i\geq E^-_{ij}$ for every place $p_i$. Therefore, an enabled firing sequence with transition-count vector $x\in\N_0^{n_T}$ satisfies $m'=m+Bx$, an important fact for the interpretation of our later derivations.

\section{Petri net random-walk encodings}\label{sec:rwpe}
We now start by presenting the first of our constructions, the Petri net random-walk positional encoding (PNRWPE). Define
\begin{align}
 d_P&=E^-\ones_{n_T} \in \mathbb{N}_0^{n_P}, & d_T&=(E^+)^\top\ones_{n_P} \in \mathbb{N}_0^{n_T}\\
 \Phi&=D_P^{\dagger}E^- \in [0,1]^{n_P \times n_T}, & \Psi&=D_T^{\dagger}(E^+)^\top \in [0,1]^{n_T \times n_P},
 \label{eq:normalization}
\end{align}
where $D_P=\diag(d_P)$ and $D_T=\diag(d_T)$. The symbol $\dagger$ denotes the reciprocal of each positive diagonal entry and zero at a zero entry. The matrix $\Phi$ describes a walker that, when at a place, chooses to move to a transition in proportion to the corresponding arc weight, and $\Psi$ has a similar interpretation for the other direction. It is worth noting that $\Phi$ and $\Psi$ have row sums in $\{0,1\}$. We would like to emphasize this last statement: a zero row stops the walk, which corresponds to the walker landing on a sink node. Define the walk transition operator as
\begin{equation}
 Q=\begin{pmatrix}0&\Phi\\\Psi&0\end{pmatrix} \in [0,1]^{n \times n}.
\end{equation}

The walker starts at a node and takes a random walk according to $Q$. The return probability after $k$ steps is the probability that the walker is back at the same node after $k$ steps, given by the diagonal of $Q^k$. Since the walk alternates between the two partitions, such a return can only happen after an even number of steps, because bipartite graphs have no odd cycles. Algebraically, this is easily seen from the fact that $\diag(Q^{2k+1})=0, \; \forall k \in \N_0$. On the other hand, for even walks we have that
\begin{equation}
 Q^2=\begin{pmatrix}\Phi\Psi&0\\0&\Psi\Phi\end{pmatrix}  \implies
 Q^{2k}=\begin{pmatrix}(\Phi\Psi)^k&0\\0&(\Psi\Phi)^k\end{pmatrix}.
 \label{eq:even}
\end{equation}
We can thus focus on the two-step operators $M_P=\Phi\Psi \in [0,1]^{n_P\times n_P}$ and $M_T=\Psi\Phi \in [0,1]^{n_T\times n_T}$, which describe the return probabilities for places and transitions, respectively. One step of $M_P$ follows a place-transition-place path, while one step of $M_T$ follows a transition-place-transition path. We are now left with a final problem, which is that of sinks, as they can simply stop the walk. To address this, we define an absorption correction that adds the missing probability mass to the diagonal of each two-step operator. This ensures that the return probabilities remain well-defined. For $a\in\{P,T\}$, define $\zeta_a=\ones-M_a\ones$ and $R_a=M_a+\diag(\zeta_a)$. Since $M_a$ is substochastic, $\zeta_a\geq0$ and furthermore $R_a\ones = M_a\ones + \zeta_a = \ones$. This correction keeps all the missing probability mass at the current node. A sink node will thus absorb the probability mass and the walker will remain at that node for all future steps. For a fixed horizon $K$, the Petri net random-walk positional encoding is defined as the vector of return probabilities after $2,4,\ldots,2K$ steps
\begin{align}
\label{eq:pnrwpe}
 r_{p_i}&=\big[(R_P)_{ii},(R_P^2)_{ii},\ldots,(R_P^K)_{ii}\big],
 r_{t_j}&=\big[(R_T)_{jj},(R_T^2)_{jj},\ldots,(R_T^K)_{jj}\big]. 
\end{align}

\subsection{Acyclic nets} \label{sec:teleport}
There is a practical use case that we must take into consideration for our positional encodings. If the directed graph representing the Petri net is acyclic, $M_a$ has zero diagonal. A walk can return under $R_a$ only by repeatedly using the added self-loop, which implies that $(R_a^k)_{ii}=(\zeta_a)_i^k$. The encodings will therefore be highly correlated on acyclic nets even when their structures differ substantially, which might limit their discriminative power. In simpler terms, variation in return signatures will be concentrated around cycles and stopping points, which can be considered both a strong and a weak point of PNRWPE, depending on the specific application.

We therefore define an additional PNRWPE variant that permits occasional jumps to nodes of the opposite type. Such jumps create fictitious return routes even when the original net has no directed cycle. This alternative is essentially an application of the PageRank walk~\cite{bianchini2005inside} adapted for place-transition alternation. At a place with outgoing arcs, the walker follows $\Phi$ with probability $\alpha$ and jumps uniformly to a transition with probability $1-\alpha$. At a transition, it follows $\Psi$ or jumps uniformly to a place. A node with no outgoing arcs always jumps, which bypasses the expressivity constraint described above. We therefore define the $P \to T$ and $T \to P$ transition matrices as
\begin{align}
 \Phi_\alpha
 &=\alpha\Phi+\frac{1}{n_T}
    (\ones_{n_P}-\alpha\Phi\ones_{n_T})\ones_{n_T}^{\top}
    \in\R^{n_P\times n_T},\label{eq:teleport-phi}\\
 \Psi_\alpha
 &=\alpha\Psi+\frac{1}{n_P}
    (\ones_{n_T}-\alpha\Psi\ones_{n_P})\ones_{n_P}^{\top}
    \in\R^{n_T\times n_P}.\label{eq:teleport-psi}
\end{align}
The deficit in parentheses is a column vector, and its outer product distributes that deficit across destinations. Note that the normalization uses the size of the destination partition, meaning that teleportation is only defined \emph{between} partitions.

Following the same logic as in the previous section, we define the two-step operators $M_{P,\alpha}=\Phi_\alpha\Psi_\alpha$ and $M_{T,\alpha}=\Psi_\alpha\Phi_\alpha$ and use them to define the teleportation PNRWPE, for a fixed horizon $K$, through
\begin{equation}\label{eq:teleport}
    Q_\alpha
    =
    \begin{pmatrix}
    0 & \Phi_\alpha \\
    \Psi_\alpha & 0
    \end{pmatrix}, \;
    Q_\alpha^{2k}
    =
    \begin{pmatrix}
    M_{P,\alpha}^{k} & 0 \\
    0 & M_{T,\alpha}^{k}
    \end{pmatrix},
    \;
    \begin{aligned}
    r_{p_i}^{\mathrm{tel}}
    &= \big[(M_{P,\alpha})_{ii}, (M_{P,\alpha}^{2})_{ii}, \ldots, (M_{P,\alpha}^{K})_{ii}\big], \\[0.5em]
    r_{t_j}^{\mathrm{tel}}
    &= \big[(M_{T,\alpha})_{jj}, (M_{T,\alpha}^{2})_{jj}, \ldots, (M_{T,\alpha}^{K})_{jj}\big].
    \end{aligned}
\end{equation}

\begin{proposition}
For $a\in\{P,T\}$, $n_P,n_T>0$, and $0<\alpha<1$, $Q_\alpha$ is irreducible with period two. $M_{a,\alpha}$ is irreducible and aperiodic, and thus has a unique stationary row distribution $\pi_a$.
\end{proposition}
\begin{proof}
By construction, every entry of $\Phi_\alpha$ is at least $(1-\alpha)/n_T>0$, and every entry of $\Psi_\alpha$ is at least $(1-\alpha)/n_P>0$, so it is possible to reach any place from a transition and any transition from a place. As previously established, the off-diagonal block structure of $Q_\alpha$ implies that returns can occur only at even times, giving the chain a period of two. Consequently, the two-step transition operator $Q_\alpha^2$ decouples into two independent subchains governed by the respective matrices $M_{a,\alpha}$. Due to the strict positivity of $\Phi_\alpha$ and $\Psi_\alpha$, their products $M_{a,\alpha}$ must have all-positive entries. Thus, each subchain is independently irreducible and aperiodic on its respective partition. By the Fundamental Theorem of Markov Chains~\cite{maltese1986simple}, each $M_{a,\alpha}$ possesses a unique stationary row distribution $\pi_a$ within the isolated partition $a \in \{P, T\}$.
\end{proof}

In summary, teleportation introduces return paths even when the original net is acyclic, allowing return signatures to reflect more than the local stopping probabilities captured by the original PNRWPE. This can yield more diverse walk signatures, but it does not guarantee that different nodes receive distinct encodings. Given that each two-step subchain converges to its stationary distribution, it is best to use short finite horizons in practice to capture transient return behavior, as long-horizon return probabilities will approach stationarity and not add any new positional information. 

\section{Diffusion-augmented message passing on Petri nets}\label{sec:da}
\subsection{Incidence diffusion}
\label{sec:incidence}
We start this section by considering the example given in Figure~\ref{fig:pair}, where both nets have one possible successor at every node. The second net produces two tokens at $p_2$ when $t_1$ fires, but this distinction disappears from $\Phi$ and $\Psi$, and hence from PNRWPE. Nevertheless, we can intuitively ``see'' that the two nets are very different: the left net conserves the total number of tokens, while the right net does not. Our goal will be to capture this difference and use it to augment message passing in a GNN operating on Petri nets.
\begin{figure}[ht]
\centering
\begin{tikzpicture}[>=Stealth,place/.style={circle,draw,minimum size=7mm},trans/.style={rectangle,draw,minimum width=3mm,minimum height=7mm},font=\small]
\foreach \x/\weight in {0/1,6/2}{
 \node[place] (p1\x) at (\x,0) {$p_1$};
 \node[trans,label=above:$t_1$] (t1\x) at (\x+1.3,0.65) {};
 \node[place] (p2\x) at (\x+2.6,0) {$p_2$};
 \node[trans,label=below:$t_2$] (t2\x) at (\x+1.3,-0.65) {};
 \draw[->] (p1\x)--(t1\x);
 \draw[->] (t1\x)--node[above right] {$\weight$}(p2\x);
 \draw[->] (p2\x)--(t2\x);
 \draw[->] (t2\x)--(p1\x);
}
\end{tikzpicture}
\caption{A change in token production that PNRWPE cannot detect. Unmarked arc weights are one. Subscripts 1 and 2 in the equations denote the left and right net, respectively.}
\label{fig:pair}
\end{figure}
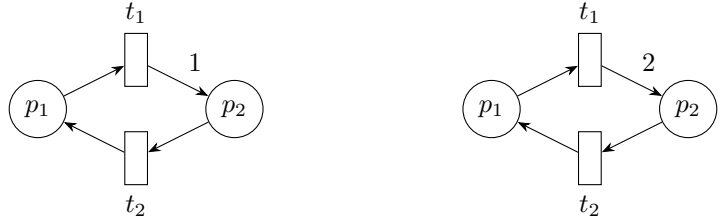

For the example above, the corresponding consumption and production matrices are
\begin{equation}
 E^-_1=E^-_2=I_2,\qquad
 E^+_1=\begin{pmatrix}0&1\\1&0\end{pmatrix},\qquad
 E^+_2=\begin{pmatrix}0&1\\2&0\end{pmatrix}.
\end{equation}
For both nets, $\Phi=I_2$ and $\Psi=\ones_{2 \times 2} - I_2$. Every node therefore has random-walk coordinates $[0,1,0,1,\ldots]$. But a clear distinction arises if we look at their incidence matrices:
\begin{equation}
 B_1=\begin{pmatrix}-1&1\\1&-1\end{pmatrix},\qquad
 B_2=\begin{pmatrix}-1&1\\2&-1\end{pmatrix},
 \qquad \det B_1=0,\quad\det B_2=-1.
 \label{eq:pair}
\end{equation}
In the left net, firing both transitions and repeating the loop has zero net effect, and the total number of tokens is conserved. In the right net, the same firing counts add one token at $p_2$. Putting this discussion in more algebraic terms, let us consider what we call a place signal $y\in\R^{n_P}$, which assigns a scalar to each place. If $B^\top y=0$, then the transition sequence that updates the net's marking (Section~\ref{subsec:petri-structure}) reveals a conservation phenomenon:
\begin{equation}
 y^\top m'=y^\top m+y^\top Bx=y^\top m.
 \label{eq:pinv}
\end{equation}
An analogous phenomenon can be seen for a transition signal $x\in\R^{n_T}$ with $Bx=0$. These conservations are called P (place) and T (transition) invariants in Petri nets, although the classical notions are usually taken with integer coefficients, as these have a direct token-count interpretation~\cite{murata}. Here we work with real vectors, as this will allow us to easily characterize the invariant subspaces and use them to augment message passing. The existence of a place-invariant vector $y\in\R^{n_P}$ such that $B^\top y=0$ implies that the weighted sum of tokens at the places is preserved by the corresponding transition sequence. The existence of a transition-invariant vector $x\in\R^{n_T}$ such that $Bx=0$ implies that the net effect of firing the transitions in the sequence described by $x$ is zero. The singular value decomposition $B=U\Sigma V^\top$ reveals the two spaces where these phenomena arise. Place-invariant modes lie in the left null space $\ker B^\top$, while transition-invariant modes lie in the right null space $\ker B$. Define the Gram matrices $L_P=BB^\top$ and $L_T=B^\top B$. These Gram matrices assign a nonnegative energy to each signal:
\begin{equation}
 y^\top L_Py=\norm{B^\top y}_2^2,\qquad
 x^\top L_Tx=\norm{Bx}_2^2.
 \label{eq:energy}
\end{equation}
For a place signal, the first energy measures the squared imbalance between weighted inputs and outputs at each transition. For a transition signal, the second measures the squared net change at the places. Thus zero energy means exact linear balance, since $\ker L_P=\ker B^\top$ and $\ker L_T=\ker B$. This observation provides an interesting notion of smoothness that we will induce in our learning procedure. The Gram matrices are symmetric positive semidefinite, and their eigenvalues are nonnegative. The zero eigenvalues correspond to invariant (the balanced/stable) subspaces, while positive eigenvalues correspond to imbalanced modes. Additionally, the eigenvectors of $L_P$ and $L_T$ can be chosen to form orthonormal bases for $\R^{n_P}$ and $\R^{n_T}$, respectively, allowing us to decompose any signal into balanced and imbalanced components. At this point the construction we are aiming at should become clear: we want to use the incidence matrix to define a diffusion operator that suppresses imbalanced components while retaining balanced ones. This will allow us to augment message passing with higher-order messages that are informed by the structural conservation capabilities of the Petri net.

Let $\gamma = 1 + \|B\|_1\|B\|_\infty$, where $\|B\|_1$ and $\|B\|_\infty$ denote the maximum absolute column sum and maximum absolute row sum of $B$, respectively. Then $\gamma > \|B\|_2^2$, since $\|B\|_2^2 = \rho(B^\top B) \le \|B^\top B\|_1 \le \|B\|_\infty\|B\|_1$ by the definition of the spectral norm, the bound of the spectral radius by any induced norm, and submultiplicativity, where $\rho(\cdot)$ denotes the spectral radius. Define the diffusion operators
\begin{equation}
 S_P=I_{n_P}-\gamma^{-1}L_P,\qquad
 S_T=I_{n_T}-\gamma^{-1}L_T.
 \label{eq:diffusion}
\end{equation}
Each one of the above operators controls a discrete diffusion process. For places, for example, the update $y^{(k+1)} = y^{(k)} - \gamma^{-1}L_Py^{(k)}$ is a gradient-descent step of size $\gamma^{-1}$ on the energy $\tfrac12 y^\top L_Py$. The same holds for transitions.

\begin{proposition}
For $a\in\{P,T\}$, the operator $S_a$ fixes every vector of $\ker L_a$ and strictly contracts its orthogonal complement. Consequently,
\begin{equation}
 \lim_{k\to\infty}S_P^k=\Pi_{\ker B^\top},\qquad
 \lim_{k\to\infty}S_T^k=\Pi_{\ker B},
 \label{eq:limit}
\end{equation}
where $\Pi$ is the orthogonal projector onto the indicated space.
\end{proposition}
\begin{proof}
Each $L_a$ is symmetric positive semidefinite with $\norm{L_a}_2=\norm{B}_2^2<\gamma$, so its eigenvalues satisfy $0\leq\lambda<\gamma$. The corresponding eigenvalues of $S_a$ are $1-\lambda/\gamma\in(0,1]$. A zero mode has eigenvalue one and is therefore fixed at every step. Every positive mode decays to zero. Applying this argument in an orthonormal eigenbasis proves the limits.
\end{proof}

\subsection{Diffusion-augmented Petri net graph neural networks}\label{sec:diffusion}
The diffusion operators describe how a signal at one node relates to signals at other nodes of the same type. For $a\in\{P,T\}$ we can diagonalize $L_a$ using the eigendecomposition $L_a=U_a\Lambda_aU_a^\top$, which implies $S_a=U_a (I_{n_a} - \gamma^{-1}\Lambda_a)U_a^\top$ and hence the entries of the kernel $S_a^k$ are
\begin{equation}
 (S_a^k)_{ij}
 =\sum_m (U_a)_{im}(U_a)_{jm}(1-\lambda_{a,m}/\gamma)^k,
 \label{eq:pair-spectral}
\end{equation}
where $\lambda_{a,m}$ is the $m$-th eigenvalue of $L_a$. Relating to what we mentioned before, $(S_a^k)_{ij}$ describes the response at node $i$ to a unit signal initially placed at node $j$ at the $k$-th diffusion step. At long times, these responses describe how the two nodes participate together in the invariant subspace. At finite times they also retain components with positive energy. This suggests using the full kernel as a relation between nodes. In practice we compute
\begin{equation}
 c_{a,ij}=\big[(S_a^2)_{ij},(S_a^4)_{ij},(S_a^{16})_{ij},
                   (S_a^{64})_{ij}\big]\in\R^4,
 \qquad a\in\{P,T\},
 \label{eq:pair-features}
\end{equation}
for every ordered pair of the same type, including self-pairs. Larger exponents suppress imbalanced modes more strongly. Each kernel is symmetric, $c_{a,ij}=c_{a,ji}$, so the kernels carry no direction. Direction remains available through the original arcs and the optional PNRWPE walk features.

The kernels described above will be used to modulate message passing, but their entries can potentially have very small values, especially off the diagonal, which can hinder learning. We therefore rescale the two sets of entries separately (for each net, node type, and exponent). Formally, let $\mathcal I_{\rm self}=\{(i,i)\}$ and $\mathcal I_{\rm off}=\{(i,j):i\ne j\}$ within type $a$. For either set $\mathcal I$, define
\begin{equation}
 s_{\mathcal I}(S_a^\tau)
 =\left(\frac{1}{|\mathcal I|}\sum_{(i,j)\in\mathcal I}
                  (S_a^\tau)_{ij}^2\right)^{1/2},\qquad
 \widehat S_{a,ij}^\tau=
 \begin{cases}
 (S_a^\tau)_{ij}/s_{\mathcal I}(S_a^\tau),&s_{\mathcal I}>0,\\
 0,&s_{\mathcal I}=0.
 \end{cases}
 \label{eq:rms}
\end{equation}
This root-mean-square (RMS) scaling uses no labels or statistics from other nets, and it preserves signs and exact zeros. Replacing the four entries in Equation~\eqref{eq:pair-features} with their rescaled values gives $\widehat c_{a,ij}$. At this point, it is worth noting an important distinction between the diffusion operators and the features derived from them. The spectral bounds and fixed spaces proved and discussed above apply to $S_a^k$. Separate rescaling of diagonal and off-diagonal entries need not preserve those properties. By using $\widehat c_{a,ij}$ to describe a pair to the model, we aim to provide a direct and meaningful signal derived from balance relations that the model can learn to use. Higher-order message passing of this form can be seen as a generalization of the original message-passing framework in Equation~\eqref{eq:msg-passing}, where the aggregation is performed over all nodes of the same type, rather than just the immediate neighbors, and the model can learn to emphasize or de-emphasize certain pairs within the same partition.

We conclude by providing the actual extension for each message-passing round $\ell$ and node type $a$. In addition to the standard procedure of Equation~\eqref{eq:msg-passing}, a DA-PNGNN adds a message from every node of the same type, modulated by a learned function of the diffusion pair features. The higher-order message to node $i$ is
\begin{equation}
 m_{a,i}^{(\ell)}
 =\sum_{j\in a}
 g\!\left(W_{g,a}^{(\ell)}\widehat c_{a,ij}+b_{g,a}^{(\ell)}\right)
 \odot
 \left(W_{{\rm src},a}^{(\ell)}h_j^{(\ell)}
       +W_{{\rm dst},a}^{(\ell)}h_i^{(\ell)}\right).
 \label{eq:pair-message}
\end{equation}
Here $h_i^{(\ell)}\in\R^d$ is the hidden state of node $i$ at round $\ell$. The matrices $W_{\rm src},W_{\rm dst}\in\R^{d \times d}$ transform node states, while $W_g\in\R^{d \times4}$ and $b_g\in\R^{d}$ turn the pair features into one coefficient per hidden channel. Parameters are separate for each type and round. Our construction therefore acts as a higher-order neural gating mechanism. In practice we apply $g$ elementwise and use either $g(u)=u$ (a linear gate) or $g(u)=\tanh(u)$ (a nonlinear gate). These choices are motivated by the fact that both activations allow signed coefficients. Every same-type node becomes available in one round, thereby including information on nodes beyond the reach of $L$ message-passing rounds along the original arcs, where $L$ is the total number of rounds. We would also like to mention that our design choices are in line with existing literature on using graph diffusion information for improved GNNs~\cite{gasteiger2019diffusion,mialon2021graphit}.

The proposed approach has a straightforward limitation. Adding the same matrix $H\geq0$ to $E^-$ and $E^+$ leaves $B$ unchanged, and reversing every arc, which replaces $B$ by $-B$, leaves both Gram matrices unchanged. Pairwise diffusion therefore describes balance relations in a symmetric way without directly relating to execution semantics and direction. This is precisely why we augment the original arc messages with the diffusion terms, making our proposed architectural design more complete.

We conclude with a note regarding graph symmetries: reordering the nodes gives $\widetilde B=\Pi_PB\Pi_T^\top$ for permutation matrices $\Pi_P,\Pi_T$. Each place operator transforms by conjugation with $\Pi_P$, and each transition operator by conjugation with $\Pi_T$. The scale $\gamma$ and the two normalizing RMS values are invariant as scalars. Walk features also reorder equivariantly with any permutation, and the same can be said about the resulting higher-order messages in Equation~\eqref{eq:pair-message}, since the sum over same-type senders does not depend on their order. This makes both contributions equivariant to permutations, and any permutation-invariant readout then gives a graph representation independent of node order.


\section{Experiments}\label{sec:experiments}
We organize our experimental section around three questions. Does Petri net structure, exposed in the two forms we propose, improve supervised models that recover properties of a net? Does the same structure help when no labels are available and the representation of the net is shaped only by its agreement with other views of the same process? Finally, which of the two constructions proves the more beneficial in whatever gain is observed, and why? We address the first two questions using two benchmarks of deliberately different character, and the third with a family of ablations that remove one ingredient at a time while holding the architecture, the parameter count, and the optimization fixed. Throughout, we call a model together with its input and configuration an \emph{arm}. Each arm is trained once per training seed and, on PDC, once per (outer) cross-validation fold. We report means and standard deviations over training seeds in order to assess average behavior and volatility.

\subsection{Experimental setting}\label{sec:setting}

\paragraph{Benchmarks.}
The first benchmark we use consists of the 96 PDC 2023 process models studied in PetriNet2Vec~\cite{colonna}.\footnote{\url{https://tinyurl.com/ydr73eck}.} Each net realizes one combination of six generation factors: long-term dependencies (A), loops (B), OR constructs (C), invisible tasks (D), optional tasks (E), and duplicate activity labels (F). Factor B has three classes and the other five have two, giving $3\cdot2^5=96$ combinations. The nets have 25 to 33 places and 28 to 43 transitions, with 65.6 nodes on average. All arcs have unit weight.

The second benchmark comes from ProcRosetta~\cite{berti}, which learns a shared latent space for three views of a process: a process tree, an event log, and a Petri net. Its generator draws random block-structured process trees of depth at most three, with at most three children per logical operator and between 2 and 30 activities, converts each tree into a workflow net with PM4Py~\cite{pm4py}, and plays out 16 traces. We experiment with two settings. The original \emph{paper} setting allows the sequence, exclusive-choice, and parallel operators and therefore yields acyclic nets. The \emph{loops} setting, which is our extension and not part of the original study, additionally allows the loop operator, so its nets contain directed cycles. Each setting has 4000 training, 512 validation, and 512 test processes, generated once with the original study's seed (13) and then frozen.\footnote{The generator of the original study is only partially seeded: its seed fixes the process trees, but the remaining two modalities depend on OS randomness through PM4Py. When reproducing, we seeded the global generator, generated each setting once, froze the files, and verified by checksum that every run used them. The process trees are therefore those determined by the original work, but the traces and the stored node order of the original study cannot be fully reproduced.} Test nets have 18.2 nodes on average (at most 58) in the paper setting and 20.7 (at most 69) with loops, again with unit arc weights throughout.

The two benchmarks are complementary by design. PDC 2023 offers a testbed in terms of what we might consider structural isomorphism classes of Petri nets, since every net differs from every other in a known and enumerable way. ProcRosetta is larger and self-supervised, with the generated nets depending on the tree structure. This implies that the quality of a net representation is measured only by how well it agrees with two other views of the same process. A model that performs well on both is unlikely to be exploiting an artifact of either.

\paragraph{PNGNNs and node inputs.}\label{sec:models}
All GNNs build on the Petri net encoder of ProcRosetta~\cite{berti}, which we refer to as the Petri net GNN (PNGNN), so that the baseline is a studied model rather than an arbitrary construction of our own. Each node $i$ has a kind $\kappa(i)\in\{\text{place},\text{visible transition},\text{silent transition}\}$ and a type $a(i)\in\{P,T\}$, and each place carries a marking feature $\mu_i\in\R^2$ holding its initial and final token count, which is zero for transitions. The encoder therefore sees node types, markings, and arcs but \emph{no activity labels by design}. Writing $\bar A\in\{0,1\}^{n\times n}$ for the directed adjacency of Equation~\eqref{eq:adj} with every positive weight set to one, the shared computation is
\begin{align}
 h_i^{(0)}&=\operatorname{Dropout}\!\left(t_{\kappa(i)}+W_\mu\mu_i+W_ee_i\right), \label{eq:gnn-input}\\
 u_i^{(\ell)}
 &=W_{\rm self}^{(\ell)}h_i^{(\ell)}
   +\sum_{j:\bar A_{ji}=1}W_{\rm in}^{(\ell)}h_j^{(\ell)}
   +\sum_{j:\bar A_{ij}=1}W_{\rm out}^{(\ell)}h_j^{(\ell)}
   +m_{a(i),i}^{(\ell)}, \label{eq:aug-update}\\
 h_i^{(\ell+1)}
 &=\operatorname{Dropout}\!\left(\operatorname{LayerNorm}\!\left(\operatorname{ReLU}(u_i^{(\ell)})\right)\right),
 \qquad \ell=0,1,2,\nonumber
\end{align}
with $t_\kappa$ a learned embedding of the node kind, $W_\mu\in\R^{d\times2}$ layer weights for the marking, $W_{\rm self},W_{\rm in},W_{\rm out}\in\R^{d\times d}$ separate weights for static, incoming, and outgoing information, and $W_e\in\R^{d\times w}$ layer weights for a $w$-dimensional positional encoding $e_i$. ReLU denotes the rectified linear unit, LayerNorm denotes layer normalization~\cite{ba2016layer}, and Dropout denotes dropout~\cite{srivastava2014dropout}. Setting $e_i=0$ and $m\equiv0$ recovers the original encoder exactly, which makes every comparison in what follows a strict addition to a fixed baseline. After three rounds, a masked mean over nodes, dropout, and a linear map give a 64-dimensional graph vector $z_G$. The ProcRosetta experiments add a second linear map for the latent variance, which is not needed for training on PDC. We follow the original study and set $d=128$.

PNGNN arms differ only in the input vector $e_i$ and in whether the same-type message $m_{a(i),i}^{(\ell)}$ is present. Each arm's bias-free projection $W_e\in\R^{d\times w}$ is exactly as wide as its input: $w=16$ for PNRWPE at $K=16$, $w=8$ for the encodings at $K=8$, $w\in\{16,32\}$ for undirected RWPE at 16 or 32 steps, and $w=1$ for out-degree. An arm without node input has no projection, so it is the upstream encoder. For a given seed every weight shared between arms starts identical, and parameter counts differ only by the projection, by at most $128\times32$. The inputs we compare are the zero vector, PNRWPE at $K=16$, the original PNRWPE truncated to $K=8$, and the teleportation PNRWPE of Equation~\eqref{eq:teleport} at $K=8$ with $\alpha\in\{0.85,0.95\}$.

DA-PNGNN adds the same-type message of Equation~\eqref{eq:pair-message} in every round, with separate parameters per type and per round. Its pair features are the rescaled kernels $\widehat c_{a,ij}$ of Equation~\eqref{eq:rms} at exponents 2, 4, 16, and 64, computed from the \emph{weighted} incidence matrix, so arc multiplicities enter the kernels even though $\bar A$ is binarized. Messages are summed over all nodes of the same type within a net. The diffusion branch draws its initial weights from a separate random stream, so for a given seed every component shared with the plain PNGNN starts bit-identical. Unless stated otherwise, reported DA-PNGNN results use the linear gate.

\paragraph{Shallow embedding methods and structural descriptors.}
Since our PDC setting is taken from PetriNet2Vec~\cite{colonna}, we compare against it as a shallow embedding method. PetriNet2Vec learns one vector per net with the distributed memory model of paragraph vectors (PV-DM)~\cite{le2014distributed}, following the reference implementation.\footnote{\url{https://github.com/juancolonna/PetriNet2Vec}. We reuse the same implementation with two exceptions: one worker thread instead of eight, so that fits are repeatable, and 500 training epochs as reported in the paper rather than the 1000 used in the repository notebooks.} Each net is a document with one sentence per transition-place-transition path, made of the two transition labels. Silent transitions share one token, and sentences whose two tokens are equal are dropped. Vectors have dimension 8 and are trained for 500 epochs with window 2, five negative samples, and minimum token count 1. A cosine 1-nearest-neighbor (1-NN) classifier then predicts each factor, making this task very closely linked to retrieval. PetriNet2Vec has no node input, so to combine it with ours we fuse its vector $v_G$ with a typewise mean of a node positional encoding by
\begin{equation}
 s_G=\left[\frac1{n_P}\sum_{p\in P}e_p\;\middle\Vert\;
           \frac1{n_T}\sum_{t\in T}e_t\right],\qquad
 z_G^+=[\nu(v_G)\Vert\nu(s_G)],
 \label{eq:fusion}
\end{equation}
where $\nu(u)=u/\norm{u}_2$ for $u\ne0$ and $\nu(0)=0$. The fusion changes which training net is nearest without altering the PetriNet2Vec fit itself. Because $s_G$ pools the node input, it has $2w$ coordinates: 32 for PNRWPE at $K=16$ and 16 for the encodings at $K=8$. Applying the same 1-NN readout to $s_G$ alone gives what we call the \emph{pooled descriptor}. To compare against a naive structural descriptor, we additionally represent each net through the following 15 statistics: the numbers of places, transitions, visible transitions, silent transitions, and arcs, the total consumed and produced multiplicity, and the mean and maximum weighted in- and out-degree of places and of transitions. These are standardized within each training fold and read out with a cosine 1-NN. Note that the statistics above include the number of silent transitions by design, which will be important for the information-access discussion below.

\paragraph{Objectives and optimization.}\label{sec:objectives}
On PDC, six linear heads predict the generation factors from $z_G$ and are trained with an equally weighted cross-entropy,
\begin{equation}
 \widehat p_f(G)=\operatorname{softmax}(W_f z_G+b_f),\qquad
 \mathcal L=-\frac{1}{6|\mathcal D_{\rm train}|}
 \sum_{G\in\mathcal D_{\rm train}}\sum_{f\in\{A,\ldots,F\}}
 y_f(G)^\top\log\widehat p_f(G),
 \label{eq:supervised}
\end{equation}
where $y_f(G)$ is the one-hot label of factor $f$. We use no label smoothing. Since ProcRosetta trains its encoder without labels, this supervised objective is our own adaptation to the PDC setting. On ProcRosetta we instead leave the objective entirely untouched. Each view $v\in\{\text{tree},\text{log},\text{net}\}$ is encoded as a diagonal Gaussian over $\R^{64}$: process trees as token sequences through a gated recurrent unit (GRU) encoder~\cite{cho2014learning}, event logs through a GRU per trace with attention pooling over traces, and Petri nets through the PNGNN followed by two heads for the mean and log-variance. One GRU decoder reconstructs the tree's token sequence from a sample of each view's latent under a grammar mask that restricts it to valid process trees, and training minimizes
\begin{equation}
 \mathcal L_{\rm PR}
 =\mathcal L_{\rm tree\to tree}+\mathcal L_{\rm log\to tree}+\mathcal L_{\rm net\to tree}
  +0.1\,\mathcal L_{\rm align}+0.1\,\mathcal L_{\rm con}+0.001\,\mathcal L_{\rm KL},
 \label{eq:pr-objective}
\end{equation}
where the three reconstruction terms are token cross-entropies with label smoothing $0.05$, $\mathcal L_{\rm align}$ averages the mean squared difference between latent means over the three pairs of views, $\mathcal L_{\rm con}$ is a symmetric InfoNCE loss~\cite{oord2018cpc} on the normalized means of each pair with temperature $0.2$ and the other processes of the batch as negatives, and $\mathcal L_{\rm KL}$ averages the Kullback-Leibler (KL) divergence of the three latent distributions from $\mathcal N(0,I)$. We change only the Petri net encoder; tokenization, the other encoders, the decoder, and every evaluation metric are taken from the reference implementation.\footnote{\url{https://github.com/fit-alessandro-berti/proc-rosetta}, commit \texttt{5a6ba64}, the last revision before the preprint~\cite{berti}.} The Petri net encoder has 166{,}656 parameters, the input projection adds 2{,}048 for PNRWPE at $K=16$ (1{,}024 for the encodings at $K=8$), and the diffusion branch another 200{,}448, so the full multimodal model grows from 524{,}713 to at most 727{,}209 parameters.

For speed and, more importantly, to keep the comparison honest, the networks in both benchmarks use the optimization settings of ProcRosetta (Table~2 in~\cite{berti}) without any tuning: AdamW~\cite{loshchilov} with learning rate $10^{-3}$ and weight decay $10^{-4}$, batch size 32, dropout $0.15$, and gradient-norm clipping at 5; the learning rate halved after two epochs without validation improvement down to a floor of $10^{-5}$; training stopped after five such epochs or 100 epochs, counting an epoch as an improvement only if it lowers the validation loss by more than $10^{-3}$; and evaluation of the checkpoint with the lowest validation loss. No arm received a hyperparameter search.

\paragraph{Evaluation protocol.}\label{sec:protocol}
On PDC we use five-fold grouped and stratified cross-validation, so that every net is predicted exactly once per training seed by a model that never saw it, with a further inner split reserving one fold for validation. Around 16 nets therefore decide when training stops and which checkpoint is kept, and the remaining 60 to 62 are used for training. The inner split is not cycled through, since no hyperparameter search is performed. The splitting follows three rules. First, the two nets of each F-pair, which a model without activity labels sees as identical, always fall in the same fold. Second, the folds are balanced on the loop factor B, the only factor with three classes: each loop class covers a third of the nets, so an unstratified split could leave a partition short of a class, and we verified empirically that balancing on B suffices for every class of every factor to occur in every partition. Third, the split is drawn once with seed 0 and shared by all models and training seeds, so every comparison uses the same folds. Our primary metric is balanced accuracy (BA), the mean over a factor's classes of the fraction of its nets classified correctly, and we summarize a model by the macro BA across the six factors. All values are means over ten training seeds, 0 to 9. The DA-PNGNN configuration and the horizon $K$ were chosen on seeds 0 to 4 alone; seeds 5 to 9 were run afterwards and are pooled with them.

An important caveat governs how these numbers may be read. Different model families have different access to information. Removing activity labels but keeping node kinds collapses the 96 nets into 48 isomorphism classes, each a pair differing only in F. Ignoring visibility as well leaves only 24 distinct structures, because D changes which transitions are silent and nothing else. The PNGNNs, which read no activity labels, therefore observe D but can never observe F, so their BA on F cannot exceed the $0.500$ of random guessing. A structural descriptor computed from arcs alone, such as the pooled walk features, observes neither D nor F. PetriNet2Vec reads activity labels and observes both. We kindly ask the reader to keep this in mind when comparing absolute scores across families in order not to mix modeling quality with information access.

On ProcRosetta the splits are those produced by the generator. All metrics come from the reference implementation and are computed on the 512 test processes from the latent means, averaged over seeds 13 to 22, with 13 being the original study's seed. Several rest on a behavioral distance between two sets of traces,
\begin{equation}
 d_{\rm beh}(X,Y)=\tfrac13\Big(\norm{p^{\rm var}_X-p^{\rm var}_Y}_1+\norm{p^{\rm df}_X-p^{\rm df}_Y}_1+\norm{p^{\rm len}_X-p^{\rm len}_Y}_1\Big),
 \label{eq:behavior}
\end{equation}
with $p^{\rm var}$, $p^{\rm df}$, and $p^{\rm len}$ the empirical distributions of trace variants, of directly-follows pairs with start and end markers, and of trace lengths. We report the aspects of the original study that directly involve the Petri net view: \emph{decoding}, where a tree is decoded greedily under the grammar mask from the latent mean $\mu_{\rm net}$ of the net view, scored by the exact-tree rate, the token edit distance normalized by the longer sequence, and $d_{\rm beh}$ between the 16 test traces and up to 16 traces played out from the decoded tree; \emph{embedding quality}, given by the Spearman correlation $\rho$ over all test pairs between embedding cosine distance and behavioral distance, and by the nearest-neighbor (NN) behavior, the mean behavioral distance between a process and its nearest embedding neighbor; \emph{cross-modal retrieval}, where each process queries the 512 test processes of another view by cosine similarity, scored by top-1 accuracy in the four directions involving the net; and the \emph{net-to-tree test loss}.

\paragraph{Controls and ablations.}
Given the plug-and-play form of our contributions, we also analyze how changing or isolating various components affects learning. \emph{Shuffled PNRWPE} permutes the walk features of each net, either within node type or across all nodes. \emph{Another net's descriptor} randomly replaces a net's pooled descriptor with that of a different net. \emph{Undirected RWPE} replaces our encoding with the standard construction of Equation~\eqref{eq:rwpe-undirected}, $\diag(Q_{\mathrm{u}}^k)$ for $k=1,\ldots,16$ or $k=1,\ldots,32$, where $Q_{\mathrm{u}}$ row-normalizes the symmetrized adjacency $A+A^\top$. \emph{Out-degree} supplies each node's summed outgoing multiplicity as a minimal structural baseline. \emph{Zero kernels} is the central control for DA-PNGNN. The diffusion kernel is set to zero while the architecture and parameters are kept, so the message of Equation~\eqref{eq:pair-message} degenerates to
\begin{equation}
 m_{a,i}^{(\ell)}
 =g\big(b_{g,a}^{(\ell)}\big)\odot
 \sum_{j\in a}
 \left(W_{{\rm src},a}^{(\ell)}h_j^{(\ell)}
       +W_{{\rm dst},a}^{(\ell)}h_i^{(\ell)}\right).
 \label{eq:zero-control}
\end{equation}
The arm retains global same-type context and the dependence of the message on the number of senders, but loses every variation of the gate across pairs. Finally, the \emph{original PNRWPE at $K=8$} serves as a width-matched reference for the teleportation features, separating the effect of the operator from that of the input width. Table~\ref{tab:grid} in Appendix~\ref{app:arms} enumerates every arm of both studies for clarity.

\paragraph{Implementation.} All experiments run on one NVIDIA GB10 system (aarch64) with Python 3.12.3 and PyTorch 2.14.0~\cite{paszke} on CUDA 13.0. Positional encodings and kernels are computed on the CPU in float64. All models train on the GPU in float32 with determinism enabled; on ProcRosetta, eight runs share the GPU through NVIDIA MPS, each with three CPU threads, whereas the original study trained on the CPU. Appendix~\ref{app:cost} reports the preprocessing and training costs for the main results obtained in this study.

\subsection{Supervised factor prediction on PDC 2023}\label{sec:results-pdc}
Table~\ref{tab:main} collects the PDC results. The main takeaway is consistent with what one would expect from theoretical considerations. The plain PNGNN is not a great reader of Petri net structure: its single worst blind spot is directed cycles, and both of our constructions improve it precisely there.

The plain PNGNN reaches a macro BA of $0.751$. Behind that average lies a striking failure: on the loop factor B it scores $0.275$, which is below the $0.333$ that can be obtained on a balanced three-class problem by random guessing. Adding PNRWPE raises macro BA by $0.044$, positive in 9 of the 10 seeds, and it does so by nearly doubling BA on B to $0.506$. A return probability under the alternating walk is, by definition, a statement about directed cycles, and B is a factor that is directly affected by them.

The diffusion kernels help considerably more. With the linear gate, DA-PNGNN gains $0.078$ over the PNGNN, positive in every seed, and with the tanh gate $0.072$. Combining either gate with PNRWPE gives the best results overall on PDC, $0.832$ and $0.831$, although the margin over the linear gate alone is within seed deviation. The combination is worth mentioning because both constructions are computed starting from the same two matrices $E^-$ and $E^+$, but the information they encode is different.

The comparison with PetriNet2Vec is also positive, but it must be read with the information-access caveat of Section~\ref{sec:setting} firmly in mind. PetriNet2Vec reaches $0.808$ while reading activity labels, which lets it observe both D and F. Our PNGNNs read no activity labels, so their BA on F is pinned at $0.500$ by construction. DA-PNGNN is nonetheless the only family that consistently exceeds PetriNet2Vec, and it does so with a handicap. Structure alone, when properly utilised, compensates globally for a modality that is never observed.

The two gate variants of DA-PNGNN trade factors against one another, with tanh ahead on B and linear ahead on C and E. The PDC runs establish no overall accuracy advantage for either. We default to the linear gate because it is markedly cheaper, as it can split the kernel channels into matrix products, whereas tanh materializes a tensor over pairs and hidden channels and uses an order of magnitude more memory (Appendix~\ref{app:cost}). For a more complete overview of performance per class, please consult Figure~\ref{fig:pdc_factor_heatmap} in Appendix~\ref{app:paired}.

\begin{table}[htbp]
\centering
\small
\setlength{\tabcolsep}{4pt}
\begin{tabular}{@{}lcc@{}}
\toprule
Model & Macro BA $\uparrow$ & BA on B, loops $\uparrow$\\
\midrule
\multicolumn{3}{@{}l}{\emph{Shallow embeddings observing every factor}}\\
\\
PetriNet2Vec & $0.808\pm0.009$ & $0.766\pm0.047$\\
PetriNet2Vec + pooled PNRWPE & $0.811\pm0.007$ & $1.000\pm0.000$\\

\midrule
\midrule

\multicolumn{3}{@{}l}{\emph{1-NN readout on a fixed descriptor, without activity labels}}\\
\\
Structural statistics & $0.722$ & $0.333$\\
Pooled PNRWPE & $0.722$ & $1.000$\\

\midrule
\midrule

\multicolumn{3}{@{}l}{\emph{Learned end to end, without activity labels}}\\
\\
PNGNN & $0.751\pm0.037$ & $0.275\pm0.071$\\
PNGNN + PNRWPE & $0.796\pm0.034$ & $0.506\pm0.097$\\
DA-PNGNN, linear & $0.830\pm0.011$ & $0.579\pm0.055$\\
DA-PNGNN, linear + PNRWPE & $\mathbf{0.832\pm0.014}$ & $0.619\pm0.045$\\
DA-PNGNN, tanh & $0.824\pm0.016$ & $\underline{0.652\pm0.033}$\\
DA-PNGNN, tanh + PNRWPE & $\underline{0.831\pm0.011}$ & $\mathbf{0.692\pm0.050}$\\
\bottomrule
\end{tabular}
\caption{PDC 2023: out-of-fold balanced accuracy, mean $\pm$ standard deviation over the ten training seeds. The two 1-NN readouts are deterministic and do not depend on the training seed. The \textbf{best} and \underline{second-best} results are marked \emph{within the end-to-end family without activity labels}, which is the only group whose members share an information budget. Supervision, information access, and fitting protocol all differ between PetriNet2Vec and the PNGNNs (Section~\ref{sec:setting}), please consult the text for context.}
\label{tab:main}
\end{table}

\paragraph{The kernel values are important.}
The most immediate objection to DA-PNGNN is that its advantage comes from the 200k additional parameters and from all-pairs connectivity rather than from anything the incidence matrix contributes. The zero-kernel control of Equation~\eqref{eq:zero-control} answers this directly, and the answer is unambiguous. Stripped of kernel content while keeping every parameter, every pairwise connection, and the same sum aggregation, the model falls to $0.686\pm0.039$ with the linear gate and $0.698\pm0.025$ with tanh. This is below the plain PNGNN that we set out to improve. Real kernels beat zero kernels by $0.13$ to $0.14$ on average, and in every one of the ten seeds in all four configurations (Appendix~\ref{app:paired}). Thus, using a dense all-pairs channel that has no signal to discriminate among its senders simply adds a term proportional to the number of same-type nodes in the messages. The structural content of $\widehat c_{a,ij}$ is what makes the higher-order scheme learn how to gate.

\paragraph{The biases make up for the receptive field.}
Figure~\ref{fig:depth} shows what happens when we increase the number of message-passing rounds $L$ from one to five for the plain PNGNN, DA-PNGNN with the linear gate, and DA-PNGNN with PNRWPE ($L=3$ reproduces Table~\ref{tab:main}). The plain PNGNN improves steadily with depth, from $0.580$ at $L=1$ to $0.794$ at $L=4$, and only with four or more rounds does its BA on B become slightly better than a random guess ($0.381$ at $L=4$). The augmented models require no depth: DA-PNGNN reaches $0.787$ with a single round, and $0.835$ with PNRWPE. Both peak at two rounds, at $0.833$ and $0.860$, the best PDC result in this study, and decline beyond, to $0.814$ and $0.811$ at $L=5$. The two biases therefore inject valuable structural information that can easily compensate for a limited receptive field of the base model. The decay in performance as the number of rounds grows beyond a certain point is in line with the oversmoothing problem of GNNs~\cite{Oono2020Graph}. We kept $L=3$ for the main results to match the original encoder in~\cite{berti}, but our findings indicate that using structural biases appropriately can render simpler and cheaper architectures effective. For added clarity, the BA values are listed in Table~\ref{tab:depth} of Appendix~\ref{app:depth}. 

\begin{figure}[htbp]
    \centering
    \includegraphics[width=\textwidth]{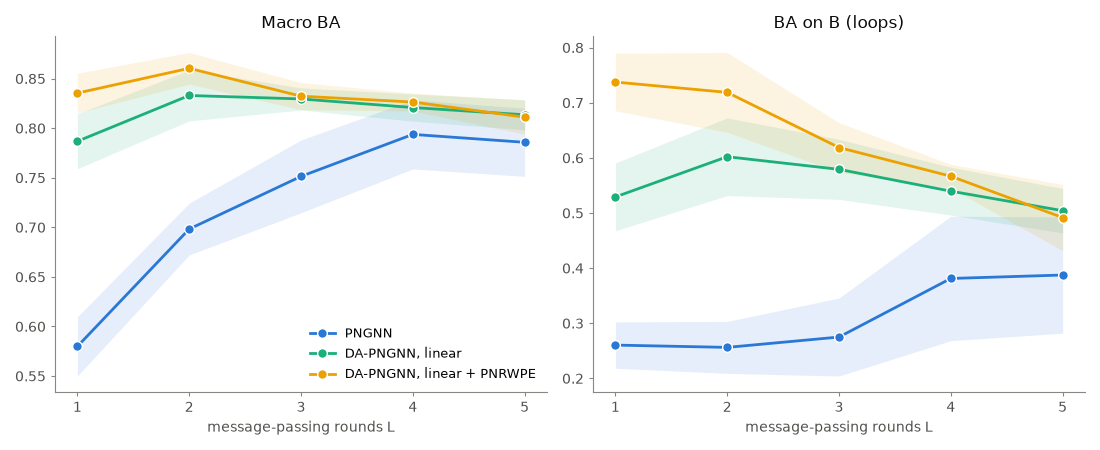}
    \caption{PDC: macro BA (left) and BA on B (right) against the number of message-passing rounds, mean and one standard deviation over ten seeds.}
    \label{fig:depth}
\end{figure}

\paragraph{Diving into the walk signatures.}
PNRWPE on PDC is strikingly degenerate at the node level, in exactly the way Section~\ref{sec:teleport} predicts. Of all nodes, $68.0\%$ carry an all-zero signature and $71.0\%$ a signature that never changes over the 16 steps, while $88.8\%$ share their signature with another node of the same type in the same net. Our empirical observations then tell us that a net of about 66 nodes has only around 14 distinct signatures. The encoding therefore almost never identifies an individual node. To see this, note that randomly shuffling signatures across all nodes still retains $+0.021\pm0.016$ of the $+0.044\pm0.032$ gain (Appendix~\ref{app:paired}). Taken at face value, this would undercut the encoding, but this is a graph-level problem and the assignment of a signature to a particular node carries little meaning. What does matter is that the descriptors are net-specific: replacing a net's pooled descriptor with another costs $0.093\pm0.007$ in BA over all seeds.

To investigate further, we use PNRWPE as a pooled descriptor which is then read with a cosine 1-NN, where it manages to separate B perfectly (Table~\ref{tab:main}). Our PE can then be best understood as a cycle detector that lives at the node level. This is precisely the information the plain PNGNN lacks on PDC, which is why so modest a feature produces such large changes. The limitations of this approach are equally instructive. The pooled descriptor reaches a macro BA of $0.722$, numerically identical to the 15 structural statistics, yet the two agree on essentially nothing: the statistics score $0.333$ on B where pooled PNRWPE scores $1.000$, and pooled PNRWPE scores $0.083$ on D and $0.500$ on F. The reason traces directly to what the walk retains. Changing D makes certain transitions silent rather than visible, and changing F renames activities so that distinct transitions share a label. Neither operation alters $E^-$ or $E^+$, so the walk matrices and every node signature are identical. We provide a rationale for the low performance on class $D$ in Appendix~\ref{app:paired}.

\paragraph{Orientation matters when walking.}
A natural question is whether any random-walk feature would have served as well. We find that this is not the case on PDC. Substituting the standard undirected RWPE changes macro BA over the plain PNGNN by only $+0.010\pm0.024$ at 16 steps and $+0.010\pm0.030$ at 32 steps, with no differences on B. The symmetrization $A+A^\top$ destroys the direction that gives a Petri net walk its meaning. A plain out-degree feature is the more serious competitor. It changes macro BA by $+0.040\pm0.052$, positive in 8 of 10 seeds, and improves B from $0.275$ to $0.425$, about two thirds of what PNRWPE does. Part of the loop signal is thus already visible in local arc counts, which the directed walk reads more reliably. In the fusion setup, pooled undirected RWPE moves PetriNet2Vec by at most $0.002$ and pooled out-degree lowers it by $0.005$.

\paragraph{Walk fusion into PetriNet2Vec trades performance on different factors.}
Fusing pooled PNRWPE into PetriNet2Vec raises A from $0.560$ to $0.805$ and B from $0.766$ to $1.000$, but lowers C from $0.896$ to $0.735$ and E from $0.628$ to $0.326$. The macro gain is therefore very small. Because Equation~\eqref{eq:fusion} normalizes both blocks, the fused cosine similarity is the average of the two cosine similarities, the nearest neighbor under this average need not be the nearest under either feature alone. This can be considered a limitation of our naive fusion rule rather than of the encoding, and it is the reason why we prefer to use the end-to-end models as primary evidence.

\subsection{Self-supervised multimodal alignment on ProcRosetta}\label{sec:results-pr}
The PDC results establish that structure helps a supervised model recover the modifications that generated a net. We now evaluate a different scenario, where training is performed without labels over a larger set of available nets and the model is later evaluated on different problems. Tables~\ref{tab:procrosetta-decode} and~\ref{tab:procrosetta-retrieval} collect the results. The main takeaway is that DA-PNGNN improves almost every metric that directly involves the net: on acyclic nets the gain concentrates in retrieval, and once loops are present it extends to decoding and the reconstruction loss as well.

\paragraph{Reproducibility of the published model.}
Before changing anything we verified that the original model we have access to is faithful. Averaged over ten seeds in the paper setting, the unmodified encoder reaches an exact-tree rate of $35.9\%$ against the published $35.9\%$, a normalized edit distance of $0.278$ against $0.282$, a behavior L1 of $0.833$ against $0.839$, a behavior correlation of $0.593$ against $0.574$, and an NN behavior of $0.935$ against $0.913$ (Table~\ref{tab:procrosetta-decode}). In retrieval (Table~\ref{tab:procrosetta-retrieval}), it achieves tree$\to$net top-1 of $45.0\%$ against $44.3\%$, net$\to$log of $14.5\%$ against $15.4\%$, and a net-to-tree test loss of $0.845$ against $0.837$. Given that the original study~\cite{berti} reports a single run, confirming reproducibility allows us to confidently make changes to the PNGNN encoder and attribute resulting differences to these changes.

\paragraph{Diffusion remains beneficial in both settings.}
In the paper setting, DA-PNGNN improves all retrieval directions involving the net as a query: net$\to$tree top-1 rises by $5.9$ points and net$\to$log by $2.7$, each in all ten seeds (Table~\ref{tab:procrosetta-retrieval}). Decoding (Table~\ref{tab:procrosetta-decode}) is not greatly affected: the exact-tree rate moves by $0.5$ points, the normalized edit distance falls by $0.009$ in 10 of 10 seeds, and behavior L1 by $0.017$ in 8 of 10. The net-to-tree loss falls by $0.009$ in 9 of 10 seeds, showing that training is slightly better. On acyclic workflow nets, the kernels sharpen where each net sits relative to the other two modalities without making the decoded tree appreciably better.

With loops the improvement extends to decoding and holds in at least 9 of 10 seeds on most measures, which is expected given that token conservation is related to cycles. Decoding improves on the exact-tree rate, by $0.7$ points to $36.0\%$ in 8 of 10 seeds, on the normalized edit distance by $0.009$, and on NN behavior by $0.012$, each in 10 of 10 seeds, while behavior L1 moves by only $0.005$. Retrieval improves in all relevant directions, by $4.1$ on net$\to$tree and $2.9$ on net$\to$log, all while the net-to-tree loss falls by $0.012$ in every seed. The fact that in this setting decoding, retrieval, and reconstruction all improve together under a configuration chosen on an entirely different benchmark and objective is a strong piece of evidence regarding the effectiveness of DA-PNGNN.

\begin{table}[htbp]
\centering
\footnotesize
\setlength{\tabcolsep}{3pt}
\begin{tabular}{@{}lcccccc@{}}
\toprule
 & \multicolumn{3}{c}{Decoding from $\mu_{\rm net}$} & \multicolumn{2}{c}{Embedding} & Loss\\
\cmidrule(lr){2-4}\cmidrule(lr){5-6}\cmidrule(l){7-7}
Encoder & Exact $\uparrow$ & Edit $\downarrow$ & L1 $\downarrow$ & $\rho\uparrow$ & NN $\downarrow$ & N$\to$T $\downarrow$\\
\midrule
\multicolumn{7}{@{}l}{\emph{Paper setting (acyclic nets)}}\\
\\
Original study, seed 13 & 35.9 & 0.282 & 0.839 & 0.574 & 0.913 & 0.837\\
PNGNN & \pmsd{35.9}{0.3} & \pmsd{0.278}{0.003} & \pmsd{0.833}{0.010} & \pmsd{0.593}{0.010} & \pmsd{0.935}{0.008} & \pmsd{0.845}{0.004}\\
+ PNRWPE & \pmsd{35.9}{0.2} & \pmsd{0.276}{0.003} & \pmsd{0.828}{0.007} & \pmsd{0.584}{0.017} & \pmsd{0.932}{0.003} & \pmsd{0.842}{0.004}\\
+ PNRWPE, shuffled & \pmsd{36.0}{0.3} & \pmsd{0.278}{0.003} & \pmsd{0.840}{0.006} & \pmsd{0.566}{0.021} & \pmsd{0.939}{0.008} & \pmsd{0.849}{0.005}\\
DA-PNGNN & \pmsd{$\mathbf{36.4}$}{0.4} & \pmsd{$\mathbf{0.269}$}{0.003} & \pmsd{$\mathbf{0.816}$}{0.010} & \pmsd{$\mathbf{0.601}$}{0.011} & \pmsd{$\underline{0.929}$}{0.008} & \pmsd{$\underline{0.836}$}{0.003}\\
DA-PNGNN + PNRWPE & \pmsd{36.2}{0.3} & \pmsd{$\mathbf{0.269}$}{0.003} & \pmsd{$\underline{0.820}$}{0.008} & \pmsd{$\underline{0.594}$}{0.015} & \pmsd{$\mathbf{0.926}$}{0.005} & \pmsd{$\mathbf{0.835}$}{0.004}\\
DA-PNGNN, zero kernels & \pmsd{$\underline{36.3}$}{0.5} & \pmsd{0.277}{0.005} & \pmsd{0.823}{0.012} & \pmsd{0.584}{0.021} & \pmsd{0.937}{0.006} & \pmsd{0.845}{0.004}\\

\midrule
\midrule

\multicolumn{7}{@{}l}{\emph{Loops setting (nets with directed cycles)}}\\
\\
PNGNN & \pmsd{35.3}{0.4} & \pmsd{0.268}{0.002} & \pmsd{0.862}{0.007} & \pmsd{0.582}{0.008} & \pmsd{0.952}{0.004} & \pmsd{0.875}{0.008}\\
+ PNRWPE & \pmsd{35.5}{0.3} & \pmsd{0.264}{0.003} & \pmsd{$\underline{0.856}$}{0.005} & \pmsd{0.580}{0.009} & \pmsd{0.949}{0.005} & \pmsd{0.872}{0.007}\\
+ PNRWPE, shuffled & \pmsd{35.2}{0.5} & \pmsd{0.272}{0.003} & \pmsd{0.867}{0.007} & \pmsd{0.575}{0.013} & \pmsd{0.953}{0.003} & \pmsd{0.882}{0.008}\\
DA-PNGNN & \pmsd{$\mathbf{36.0}$}{0.5} & \pmsd{$\mathbf{0.259}$}{0.004} & \pmsd{0.857}{0.014} & \pmsd{$\underline{0.583}$}{0.010} & \pmsd{$\mathbf{0.940}$}{0.006} & \pmsd{$\underline{0.863}$}{0.006}\\
DA-PNGNN + PNRWPE & \pmsd{$\underline{35.9}$}{0.4} & \pmsd{$\mathbf{0.259}$}{0.006} & \pmsd{$\mathbf{0.855}$}{0.005} & \pmsd{$\mathbf{0.589}$}{0.013} & \pmsd{$\underline{0.941}$}{0.004} & \pmsd{$\mathbf{0.861}$}{0.006}\\
DA-PNGNN, zero kernels & \pmsd{35.4}{0.9} & \pmsd{0.274}{0.005} & \pmsd{0.875}{0.009} & \pmsd{0.568}{0.009} & \pmsd{0.962}{0.005} & \pmsd{0.879}{0.006}\\
\bottomrule
\end{tabular}
\caption{ProcRosetta: decoding, embedding quality, and test loss of the net view under each encoder, on the 512 test processes. Each cell gives the mean over seeds 13 to 22 with the sample standard deviation. Decoding is from $\mu_{\rm net}$; N and T denote the net and tree views, and the last column is the net-to-tree test loss. The first row is the single published run of the original study, it is shown here for reference only. The \textbf{best} and \underline{second-best} values of each setting are marked. PNRWPE uses $K=16$, and the shuffle permutes it across all nodes of each net. Retrieval for the same runs is in Table~\ref{tab:procrosetta-retrieval}.}
\label{tab:procrosetta-decode}
\end{table}

\begin{table}[htbp]
\centering
\footnotesize
\setlength{\tabcolsep}{6pt}
\begin{tabular}{@{}lcccc@{}}
\toprule
 & \multicolumn{4}{c}{Retrieval top-1 (\%) $\uparrow$}\\
\cmidrule(l){2-5}
Encoder & N$\to$T & T$\to$N & N$\to$L & L$\to$N\\
\midrule
\multicolumn{5}{@{}l}{\emph{Paper setting (acyclic nets)}}\\
\\
Original study, seed 13 & 34.2 & 44.3 & 15.4 & 14.1\\
PNGNN & \pmsd{34.6}{1.3} & \pmsd{45.0}{1.1} & \pmsd{14.5}{1.1} & \pmsd{13.5}{0.9}\\
+ PNRWPE & \pmsd{36.0}{1.2} & \pmsd{45.6}{1.6} & \pmsd{14.7}{1.4} & \pmsd{13.4}{0.6}\\
+ PNRWPE, shuffled & \pmsd{33.3}{0.9} & \pmsd{43.8}{1.2} & \pmsd{13.8}{1.0} & \pmsd{12.9}{0.7}\\
DA-PNGNN & \pmsd{$\mathbf{40.5}$}{1.4} & \pmsd{$\underline{46.0}$}{1.2} & \pmsd{$\mathbf{17.2}$}{0.7} & \pmsd{$\mathbf{16.1}$}{1.2}\\
DA-PNGNN + PNRWPE & \pmsd{$\underline{39.8}$}{1.1} & \pmsd{45.9}{1.4} & \pmsd{$\underline{17.1}$}{1.0} & \pmsd{$\underline{15.9}$}{0.8}\\
DA-PNGNN, zero kernels & \pmsd{35.0}{2.0} & \pmsd{$\mathbf{46.3}$}{1.5} & \pmsd{15.4}{0.9} & \pmsd{14.6}{1.1}\\

\midrule
\midrule

\multicolumn{5}{@{}l}{\emph{Loops setting (nets with directed cycles)}}\\
\\
PNGNN & \pmsd{45.5}{1.6} & \pmsd{52.6}{1.7} & \pmsd{15.8}{1.2} & \pmsd{15.4}{1.0}\\
+ PNRWPE & \pmsd{47.1}{2.0} & \pmsd{52.3}{1.7} & \pmsd{17.1}{0.9} & \pmsd{15.7}{0.8}\\
+ PNRWPE, shuffled & \pmsd{44.5}{1.6} & \pmsd{51.0}{2.1} & \pmsd{16.4}{1.4} & \pmsd{15.7}{0.9}\\
DA-PNGNN & \pmsd{$\mathbf{49.6}$}{2.4} & \pmsd{$\mathbf{54.6}$}{1.2} & \pmsd{$\mathbf{18.7}$}{1.4} & \pmsd{$\mathbf{16.3}$}{0.8}\\
DA-PNGNN + PNRWPE & \pmsd{$\underline{49.4}$}{1.4} & \pmsd{$\underline{54.0}$}{1.3} & \pmsd{$\underline{18.4}$}{1.5} & \pmsd{$\underline{16.1}$}{1.9}\\
DA-PNGNN, zero kernels & \pmsd{44.9}{1.9} & \pmsd{51.3}{1.5} & \pmsd{16.4}{1.2} & \pmsd{15.1}{1.5}\\
\bottomrule
\end{tabular}
\caption{ProcRosetta: cross-modal retrieval top-1 accuracy with the Petri net modality, for the same runs as Table~\ref{tab:procrosetta-decode}. Each cell gives the mean over seeds 13 to 22 with the sample standard deviation. N, T, and L denote the net, tree, and log views, and X$\to$Y reads as querying with X against candidates of type Y. Higher is better throughout. The first row is the single published run of the original study, it is shown here for reference only. The \textbf{best} and \underline{second-best} values of each setting are marked.}
\label{tab:procrosetta-retrieval}
\end{table}

\paragraph{The diffusion control isolates the kernel content.}
In the paper setting the full model beats the zero-kernel control on net$\to$tree by $5.5$ points in all ten seeds, on net$\to$log by $1.8$, on behavior $\rho$ by $0.017$, and on normalized edit distance by $0.008$, each in at least 9 of 10 seeds, while the control itself sits at the baseline (net$\to$tree $+0.4$). With loops the control again sits at the baseline (net$\to$tree $-0.6$). In both settings, when the net is the query, the improvement is attributable to kernel content rather than to giving every node access to every same-type node. In the loops setting, the full model beats the control in at least 9 of 10 seeds on the edit distance, behavior L1, $\rho$, NN behavior, net$\to$tree, tree$\to$net, net$\to$log, and the loss, and in 8 of 10 on the exact-tree rate.

\paragraph{PNRWPE degenerates on small acyclic nets.}
This is the confirmation of the claim made in Section~\ref{sec:teleport}. On the 4000 training nets of the paper setting, every PNRWPE signature is constant over the 16 steps, $87.9\%$ all-zero and $12.1\%$ all-one, leaving a net with four distinct signatures on average, two per node type, which merely mark where the walk stops. As explained previously, with no directed cycles $M_a$ has zero diagonal, so a return under $R_a$ can only happen through the fictitious self-loop, thus $(R_a^k)_{ii}=(\zeta_a)_i^k$. The learned models behave accordingly: PNRWPE moves net$\to$tree by $1.4$ points and every other metric by less than the spread across seeds. Once loops are allowed, the signatures vary: at $K=16$, $43.8\%$ of nodes are all-zero and $54.0\%$ constant, with 7.8 distinct signatures per net. PNRWPE then yields small but consistent improvements: the normalized edit distance by $0.004$, behavior L1 by $0.006$, and NN behavior by $0.003$, each in 9 of 10 seeds, and the net-to-tree loss by $0.003$ in 8 of 10, while $\rho$ and retrieval stay within noise.

The shuffle control needs a word of care, because it is the control one naturally reaches for and it answers a weaker question than it appears to. PNRWPE does beat its own shuffle in both settings, but the shuffled configuration is itself worse than the plain baseline, by $1.3$ and $1.0$ points on net$\to$tree. The comparison therefore establishes that these node signatures act as noise, not that correctly placed ones help beyond supplying none at all. We state this because the two readings are easy to conflate.

\paragraph{Teleportation removes degeneracy without adding utility.}
The teleportation variant was derived as a remedy for the degeneracy of PNRWPE on acyclic nets, so we evaluate it here. We use two jump parameters, $\alpha=0.85$ and $\alpha=0.95$, and set $K=8$ following the short-horizon recommendation of Section~\ref{sec:teleport}. We verified separately that little is lost by it, since coordinates 9 to 16 of the original encoding separate almost none of the nodes that already collide on coordinates 1 to 8. Every arm with an original PNRWPE input is rerun with the teleportation features at the same input width, under the same protocol, seeds, and controls, with the original encoding at $K=8$ added as a width-matched reference. For a given seed both $\alpha$ arms start from the same weights and share their shuffle permutations.

\begin{table}[htbp]
\centering
\small
\setlength{\tabcolsep}{3.5pt}
\begin{tabular}{@{}llccccc@{}}
\toprule
Setting & Encoding & All-zero & Constant & Colliding & Distinct & 1st dir.\\
\midrule
Paper & original & $87.9\%$ & $100\%$ & $86.9\%$ & 4.0 & $100\%$\\
 & teleport, $\alpha=0.85$ & $0$ & $5.7\%$ & $46.3\%$ & 12.3 & $97.9\%$\\
 & teleport, $\alpha=0.95$ & $0$ & $5.7\%$ & $46.4\%$ & 12.3 & $95.9\%$\\
\midrule
Loops & original & $45.7\%$ & $55.8\%$ & $82.8\%$ & 7.6 & $89.9\%$\\
 & teleport, $\alpha=0.85$ & $0$ & $4.9\%$ & $36.1\%$ & 15.8 & $93.0\%$\\
 & teleport, $\alpha=0.95$ & $0$ & $4.9\%$ & $36.1\%$ & 15.8 & $79.0\%$\\
\bottomrule
\end{tabular}
\caption{Feature diagnostics at $K=8$ on the 4000 training nets of each ProcRosetta setting. \emph{All-zero} and \emph{Constant}: nodes whose eight coordinates are all zero, or all equal. \emph{Colliding}: nodes sharing their signature, to ten decimal places, with another node of the same type in the same net. \emph{Distinct}: distinct signatures per net, counted per node type and summed. \emph{1st dir.}: share of the variance of the eight coordinates, over all nodes, along the first principal direction. Teleportation eliminates degeneracy on every count, yet the \emph{Colliding} and \emph{Distinct} columns are identical across the two values of $\alpha$. Slower mixing changes the values of the encoding, but not node separability.}
\label{tab:teleport-diag}
\end{table}

At the level of the features, teleportation does everything it was designed to do (Table~\ref{tab:teleport-diag}). All-zero signatures disappear entirely, constant signatures fall to a few percent, and the share of colliding nodes drops from roughly $85\%$ to under half. The number of distinct signatures per net roughly triples, from $4.0$ to $12.3$ in the paper setting and doubles from $7.6$ to $15.8$ in the loops setting. As far as every diagnostic one would write down in advance goes, the fix works.

It nonetheless does not translate to downstream performance, and three further diagnostics explain why. First, the two values of $\alpha$ separate exactly the same nodes: in more than $99.5\%$ of the ProcRosetta training nets, precisely the same nodes share a signature at $\alpha=0.85$ and at $\alpha=0.95$. Slower mixing changes the values the encoding takes, not which nodes it can tell apart. Second, the variation teleportation does add is practically one-dimensional. At $\alpha=0.85$, a principal component analysis reveals that the first principal direction carries $97.9\%$ of the variance of the eight coordinates in the paper setting and $93.0\%$ with loops. This being the case, the additional positional coordinates in fact contain no real value. Third, and most fundamentally, the extra return routes are fictitious by construction, and they report that a walker could reach a node again if permitted to teleport, which has no counterpart in the semantics of a Petri net. We therefore expect that a learner will not find anything useful in the extra variation, and that the model will not improve on the original encoding. This finding discourages PageRank-like walks in this domain.

We report detailed results that closely follow these diagnostics in Table~\ref{tab:teleport-pr} of Appendix~\ref{app:teleport}. Against no positional input, teleportation changes every decoding and retrieval metric by at most two points or about $0.01$. In the paper setting it gains $1.2$ to $1.8$ points on net$\to$tree and lowers the net-to-tree loss by $0.005$ to $0.006$, in 7 to 9 of 10 seeds, while in both settings it lowers the behavior correlation $\rho$ by $0.010$ to $0.016$, in 7 to 9 of 10 seeds. Against the width-matched original encoding the gains shrink to within a point and the loss of $\rho$ remains, reaching $0.018$ to $0.023$ in all ten seeds with loops. In summary, teleportation is a correct fix for the problem it targets and an ineffective one for the problem we care about. It restores variation to a degenerate feature, but the variation it provides is superfluous and not useful to the PNGNN. We report this negative result at length because it shows how a property of an encoding is not always the same thing as its value as an input to a learning model, and this relationship is task-dependent.

\section{Conclusion}\label{sec:conclusion}

\subsection{Limitations and future work}\label{sec:scope}
Our evidence rests on two synthetic families of nets, and it is worth being precise about what that supports. The two benchmarks measure different things: the supervised recovery of the parameters that generated a net, and the geometry of a self-supervised latent space shared with two other views of the same process. Whether either construction helps other self-supervised objectives, or process-similarity tasks evaluated by behavior rather than by generation factors, remains to be tested.

There are two additional caveats regarding DA-PNGNN that we want to name rather than bury. The first is that both benchmarks have unit arc weights, so the multiplicity distinction of Figure~\ref{fig:pair}, which motivated the kernels, is never tested explicitly in our experiments. The kernels are computed from the weighted incidence matrix and would register that distinction, but whether doing so helps a model remains untested. The second caveat concerns scaling. Pair kernels need $O(n_P^2+n_T^2)$ storage per exponent and dense matrix powers cost cubic time, which is cheap for nets of tens of nodes, such as ours, but not for industrial process models, where sparse structures or a low-rank factorization of $S_a^k$ would be needed. Figure~\ref{fig:cost} in Appendix~\ref{app:cost} places every variant by its epoch time and memory: the linear gate costs 2.6 times the plain encoder's epoch on PDC and 4 to 5\% more on ProcRosetta, where the encoder is a small part of the joint model, while the tanh gate needs 2.4 to 17 times the memory.

\subsection{Concluding remarks}\label{sec:concluding-remarks}
We asked whether the structure that Petri net theory makes explicit can improve learned representations of these formal process models. To this end, we exposed two kinds of structure to GNNs. The first captures where each node sits relative to directed cycles, through PNRWPE, and the second how the incidence matrix balances token flow, through the diffusion kernels of DA-PNGNN. The kernels improved results on both benchmarks: the supervised recovery of generation factors on PDC 2023, and the self-supervised multimodal objective of ProcRosetta, where the configuration chosen on PDC was used without retuning. A control that zeroes the kernels while keeping every parameter and every same-type connection lost the entire gain, so the improvement comes from the kernel values rather than from added capacity or connectivity. The walk encoding helped most where nets contain directed cycles and, as its derivation predicts, carried almost no information on acyclic nets. We see this as an instance of a broader pattern: when data arrives with its own established symmetries, useful inductive biases are already encoded in its geometry, and the work is to find an operator that carries them inside the learning model. Using this modern geometric deep learning approach for learning on Petri nets is, in our view, a promising and still largely unexplored direction at the intersection of machine learning and process mining.

\bibliographystyle{alpha} 
\bibliography{refs}

\clearpage

\appendix

\section{Experimental arms}\label{app:arms}
Table~\ref{tab:grid} enumerates every arm of both studies. On PDC, each arm trains one model per fold and per seed; on ProcRosetta, each arm trains one joint multimodal model per seed and setting.

\begin{table}[h]
\centering
\small
\begin{tabular}{@{}p{0.2\linewidth}p{0.55\linewidth}p{0.19\linewidth}@{}}
\toprule
Model & Node input or variant & Seeds\\
\midrule
\multicolumn{3}{@{}l}{\emph{PDC 2023: supervised factor prediction and the PetriNet2Vec setting}}\\
\\
PNGNN & no PE; PNRWPE ($K=16$); PNRWPE shuffled within type or across nodes & 0-9\\
PNGNN & undirected RWPE (16 or 32 steps); out-degree & 0-9\\
DA-PNGNN & linear or tanh gate, sum aggregation, rescaled kernels; real or zero kernels; no PE or PNRWPE & 0-9\\
Depth & PNGNN, DA-PNGNN (linear gate), and DA-PNGNN + PNRWPE at 1 to 5 message-passing rounds & 0-9\\
PetriNet2Vec & alone; fused with pooled PNRWPE or with another net's descriptor & 0-9\\
PetriNet2Vec & fused with pooled undirected RWPE or pooled out-degree & 0-9\\
1-NN only & pooled PNRWPE, undirected RWPE, and out-degree; structural statistics & deterministic\\
\midrule
\multicolumn{3}{@{}l}{\emph{ProcRosetta: self-supervised objective, paper and loops settings}}\\
\\
PNGNN & original; with PNRWPE ($K=16$); with PNRWPE shuffled across nodes & 13-22\\
DA-PNGNN & linear gate; no PE, PNRWPE ($K=16$), or zero kernels & 13-22\\
Teleportation & original PNRWPE ($K=8$) or teleportation PNRWPE ($\alpha\in\{0.85,0.95\}$); teleportation PNRWPE shuffled across nodes; DA-PNGNN with either encoding & 13-22\\
\bottomrule
\end{tabular}
\caption{Every arm run on the two benchmarks. Teleportation PNRWPE was additionally run on PDC under the same protocol, but we omit those arms since the construction targets acyclic nets and ProcRosetta's paper setting is the setting it was designed for.}
\label{tab:grid}
\end{table}

\section{Additional results on PDC}\label{app:paired}
Because every arm of a given seed shares its initialization with its comparators, we provide the paired comparisons discussed in Section~\ref{sec:results-pdc} in Table~\ref{tab:paired}. The kernel effect is the clearest signal in the table: real kernels beat the zero-kernel control in all ten seeds, in all four configurations, by between $0.126$ and $0.143$. Additionally, we provide a summary of the results per factor on PDC for clarity in Figure~\ref{fig:pdc_factor_heatmap}, as well as a study of the effect of the PNRWPE walk length on the performance of the method in Figure~\ref{fig:pdc_pnrwpe_horizon}. The latter shows that the gain is in place by $K=8$ and that a longer walk of $K=24$ adds nothing, which motivated fixing $K=16$ for every experiment except the width-matched teleportation comparisons.

We also motivate the low performance of PNRWPE as a pooled descriptor for factor D that can be seen in Figure~\ref{fig:pdc_factor_heatmap}. Such a result is expected and can be explained by the fold structure we use for validation, since nets differing only in D are not grouped together (unlike for F). For our split, in 80 of the 96 predictions the query's exact arc-duplicate sits in the candidate pool at cosine distance zero, but with the opposite D label. The 1-NN copies it, and all 80 are wrong. The other 16 have the twin held out with them and land at chance, giving 88 of 96 misclassifications and a balanced accuracy of $0.083$.

\begin{table}[htbp]
\centering
\small
\setlength{\tabcolsep}{4pt}
\begin{tabular}{@{}lc@{}}
\toprule
Paired comparison (first $-$ second) & Over 10 seeds\\
\midrule
PNGNN + PNRWPE $-$ PNGNN & $+0.044\pm0.032\ (9)$\\
PNGNN + PNRWPE $-$ same, shuffled across nodes & $+0.023\pm0.026\ (9)$\\
PNGNN + PNRWPE $-$ same, shuffled within type & $+0.025\pm0.027\ (8)$\\
PetriNet2Vec + pooled PNRWPE $-$ PetriNet2Vec & $+0.003\pm0.009\ (6)$\\
Fusion: own $-$ another net's descriptor & $+0.093\pm0.007\ (10)$\\
\midrule
DA-PNGNN, linear $-$ PNGNN & $+0.078\pm0.041\ (10)$\\
DA-PNGNN, tanh $-$ PNGNN & $+0.072\pm0.042\ (9)$\\
DA-PNGNN, linear + PNRWPE $-$ PNGNN + PNRWPE & $+0.036\pm0.037\ (9)$\\
DA-PNGNN, tanh + PNRWPE $-$ PNGNN + PNRWPE & $+0.035\pm0.029\ (9)$\\
Linear: real $-$ zero kernels & $+0.143\pm0.044\ (10)$\\
Tanh: real $-$ zero kernels & $+0.126\pm0.034\ (10)$\\
Linear + PNRWPE: real $-$ zero kernels & $+0.139\pm0.038\ (10)$\\
Tanh + PNRWPE: real $-$ zero kernels & $+0.133\pm0.050\ (10)$\\
DA-PNGNN, linear: + PNRWPE $-$ no PE & $+0.003\pm0.010\ (7)$\\
DA-PNGNN, tanh: + PNRWPE $-$ no PE & $+0.007\pm0.016\ (6)$\\
\bottomrule
\end{tabular}
\caption{PDC: paired differences in macro BA, mean $\pm$ standard deviation over the ten seeds, with the number of seeds in which the difference is positive given in parentheses. Rows whose labels do not mention PNRWPE use no node input.}
\label{tab:paired}
\end{table}

\begin{figure}[htbp]
    \centering
    \includegraphics[width=0.8\textwidth]{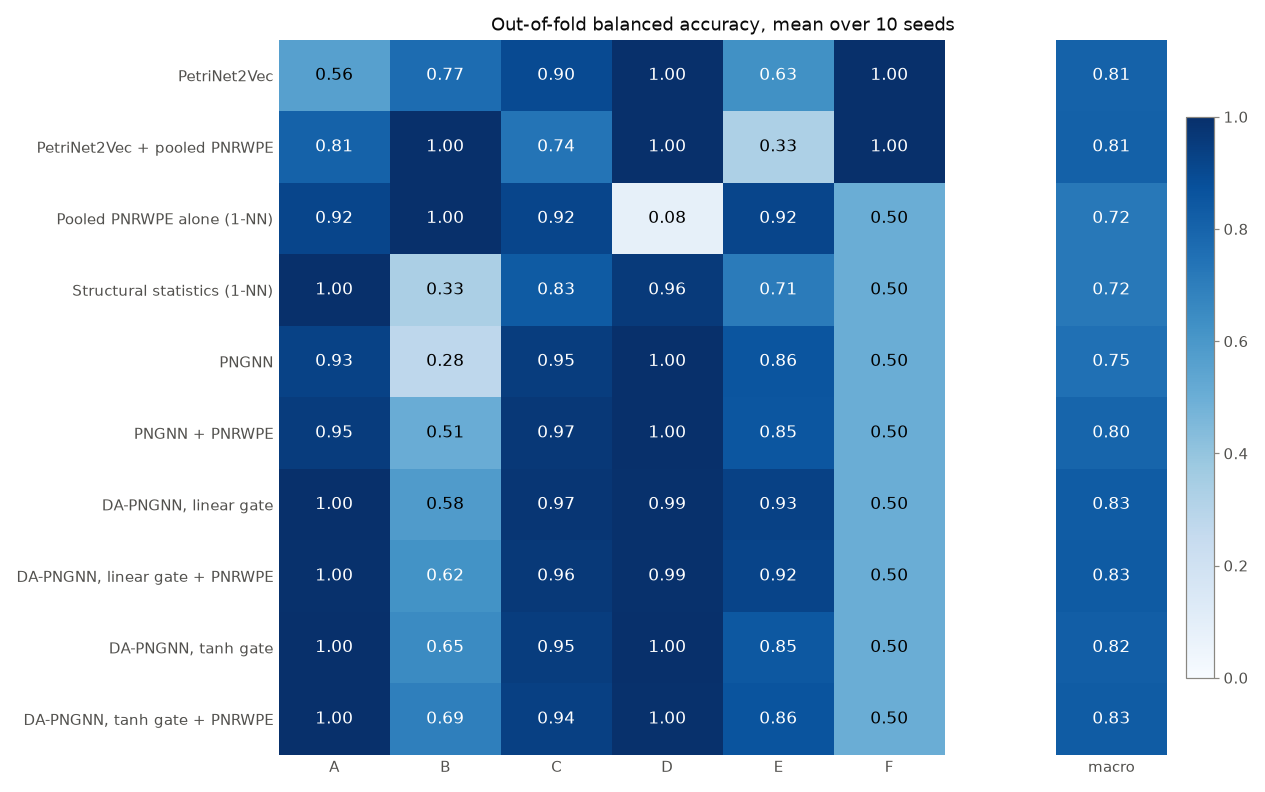}
    \caption{PDC: heatmap of the average BA for each factor. Best viewed in color.}
    \label{fig:pdc_factor_heatmap}
\end{figure}

\begin{figure}[htbp]
    \centering
    \includegraphics[width=0.8\textwidth]{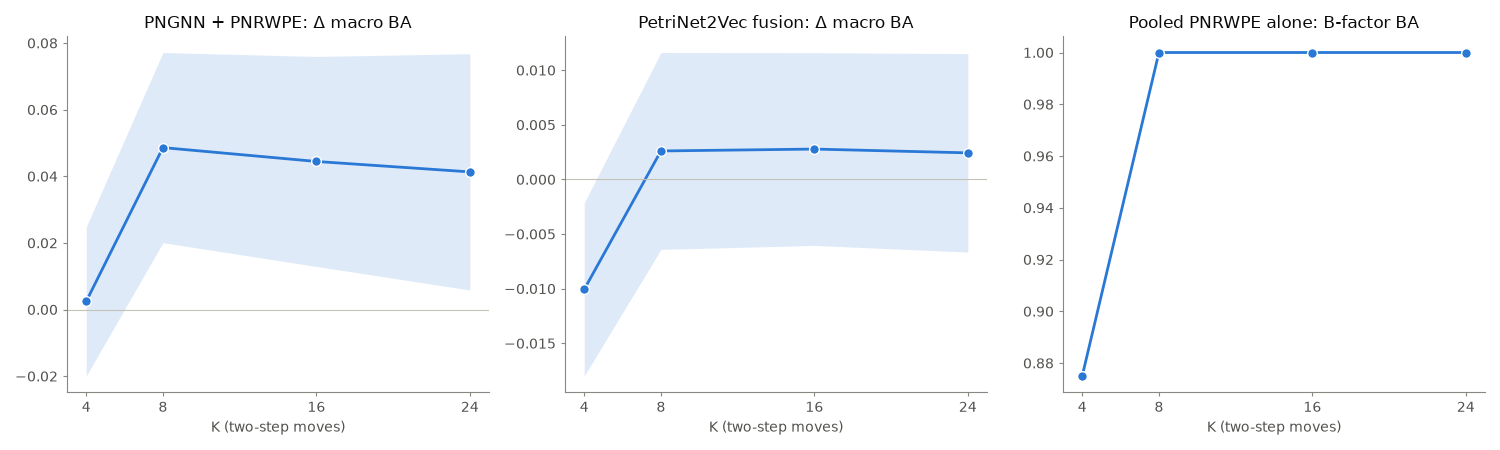}
    \caption{PDC: effect of the walk length $K$. Left: gain in macro BA of PNGNN + PNRWPE over PNGNN. Middle: the same gain for PetriNet2Vec fused with pooled PNRWPE. Right: BA on B of the pooled PNRWPE descriptor. Shaded bands show one standard deviation over ten seeds.}
    \label{fig:pdc_pnrwpe_horizon}
\end{figure}

\section{Depth of message passing}\label{app:depth}
Table~\ref{tab:depth} lists the values behind Figure~\ref{fig:depth} of Section~\ref{sec:results-pdc}: the plain PNGNN, DA-PNGNN with the linear gate, and DA-PNGNN with PNRWPE at one to five message-passing rounds, with everything else as in the main protocol.

\begin{table}[htbp]
\centering
\small
\setlength{\tabcolsep}{4pt}
\begin{tabular}{@{}lccc@{}}
\toprule
$L$ & PNGNN & DA-PNGNN & DA-PNGNN + PNRWPE\\
\midrule
\multicolumn{4}{@{}l}{\emph{Macro BA $\uparrow$}}\\
1 & $0.580\pm0.030$ & $0.787\pm0.028$ & $0.835\pm0.020$\\
2 & $0.698\pm0.026$ & $0.833\pm0.026$ & $0.860\pm0.016$\\
3 & $0.751\pm0.037$ & $0.830\pm0.011$ & $0.832\pm0.014$\\
4 & $0.794\pm0.035$ & $0.821\pm0.014$ & $0.826\pm0.009$\\
5 & $0.786\pm0.034$ & $0.814\pm0.015$ & $0.811\pm0.017$\\
\midrule
\multicolumn{4}{@{}l}{\emph{BA on B, loops $\uparrow$}}\\
1 & $0.260\pm0.042$ & $0.529\pm0.061$ & $0.738\pm0.052$\\
2 & $0.256\pm0.047$ & $0.602\pm0.070$ & $0.719\pm0.072$\\
3 & $0.275\pm0.071$ & $0.579\pm0.055$ & $0.619\pm0.045$\\
4 & $0.381\pm0.113$ & $0.540\pm0.043$ & $0.567\pm0.022$\\
5 & $0.387\pm0.105$ & $0.504\pm0.040$ & $0.492\pm0.060$\\
\bottomrule
\end{tabular}
\caption{PDC: macro BA (top) and BA on B (bottom), mean $\pm$ standard deviation over ten seeds, for $L=1$ to $5$ message-passing rounds. The rows at $L=3$ reproduce Table~\ref{tab:main}.}
\label{tab:depth}
\end{table}

\section{Computational cost}\label{app:cost}
Table~\ref{tab:cost} reports PDC training cost. The two gates differ sharply in memory, and the reason is structural rather than incidental: the linear gate can split its four kernel channels into matrix products and never forms a tensor over pairs and hidden channels, whereas the tanh gate must materialize that tensor. Since the PDC runs establish no accuracy advantage for tanh, we treat the linear gate as the default. On ProcRosetta, where eight runs share the GPU, an epoch of the runs in Section~\ref{sec:results-pr} takes 61 s with the PNGNN and 64 s with DA-PNGNN in the paper setting, and 68 against 69 s with loops. On the CPU in float64, PNRWPE ($K=16$) takes 0.15 ms per PDC net on average (maximum 0.21 ms) and 0.08 ms per ProcRosetta net (maximum 0.20 ms). The rescaled pair kernels take 0.16 ms per PDC net (maximum 0.21 ms) and 0.07 ms per ProcRosetta net (maximum 0.24 ms). These numbers should not be read as evidence that the method scales, since the nets here are small, which hides asymptotics. Figure~\ref{fig:cost} measures every variant alone on the GPU for one warm-up epoch followed by three timed epochs of training and validation, whereas Table~\ref{tab:cost} averages whole runs, whose first epochs carry start-up costs. On PDC the linear gate multiplies the epoch time of the plain encoder by 2.6 and its memory by 1.9, and the tanh gate by 10 and 17. On ProcRosetta the Petri net encoder is a small part of the joint model: the linear gate adds 4 to 5\% to an epoch and little memory, the tanh gate 11\% and 2.4 to 4.3 times the memory. PNRWPE adds a few milliseconds per epoch on PDC, at most 6\% on ProcRosetta, and no memory.

\begin{table}[h]
\centering
\small
\begin{tabular}{@{}lcccc@{}}
\toprule
Model & s / epoch & s / run & Epochs & Peak memory (MB)\\
\midrule
PNGNN & $0.027$ & $1.7$ & $63$ & 82\\
DA-PNGNN, linear & $0.079$ & $4.4$ & $55$ & 159\\
DA-PNGNN, tanh & $0.152$ & $8.0$ & $53$ & 1363\\
\bottomrule
\end{tabular}
\caption{PDC training cost on the GPU without node input: seconds per epoch and per run, and epochs until early stopping, averaged over the runs of the ten seeds (one per fold and seed). With PNRWPE the time per epoch changes by at most $0.007$ s and the peak memory by at most 1 MB.}
\label{tab:cost}
\end{table}

\begin{figure}[H]
    \centering
    \includegraphics[width=\textwidth]{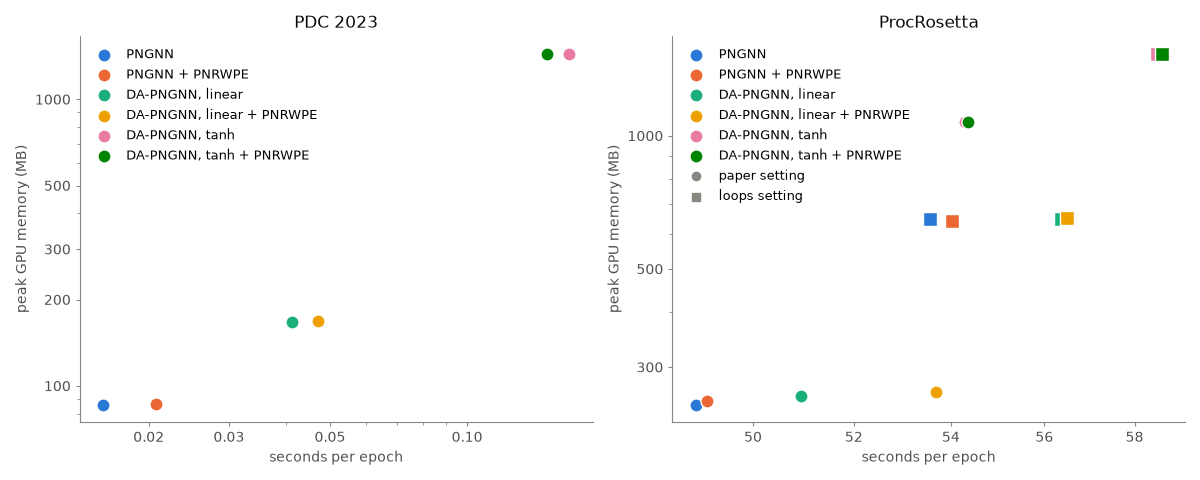}
    \caption{Seconds per epoch against peak GPU memory of every PNGNN variant, each trained alone on the GPU: one warm-up epoch, then the mean over three epochs, each a training and validation pass, on logarithmic axes. Left: PDC 2023. Right: ProcRosetta, where the joint multimodal model is timed, in the paper (circles) and loops (squares) settings.}
    \label{fig:cost}
\end{figure}

\clearpage
\section{Teleportation results on ProcRosetta}\label{app:teleport}
Table~\ref{tab:teleport-pr} lists the paired comparisons summarized in Section~\ref{sec:results-pr}, for both values of $\alpha$ in both ProcRosetta settings.

\begin{table}[H]
\centering
\footnotesize
\setlength{\tabcolsep}{4pt}
\begin{tabular}{@{}lcccccc@{}}
\toprule
Comparison (first $-$ second) & Exact (pp) & L1 & $\rho$ & N$\to$T (pp) & T$\to$N (pp) & Loss\\
\midrule
\multicolumn{7}{@{}l}{\emph{Paper setting, $\alpha=0.85$}}\\
Teleport $-$ no PE & $+0.25\,(6)$ & $-0.011\,(7)$ & $-0.015\,(1)$ & $+1.8\,(8)$ & $+0.6\,(4)$ & $-0.006\,(9)$\\
Teleport $-$ original $K=8$ & $+0.16\,(5)$ & $-0.006\,(7)$ & $-0.008\,(3)$ & $+0.7\,(7)$ & $+0.4\,(7)$ & $-0.004\,(9)$\\
Teleport $-$ shuffled & $+0.33\,(8)$ & $-0.010\,(8)$ & $+0.015\,(9)$ & $+2.9\,(10)$ & $+1.1\,(7)$ & $-0.007\,(10)$\\
Shuffled $-$ no PE & $-0.08\,(4)$ & $-0.002\,(7)$ & $-0.030\,(0)$ & $-1.1\,(3)$ & $-0.5\,(3)$ & $+0.002\,(5)$\\
DA + teleport $-$ DA & $-0.18\,(2)$ & $-0.007\,(8)$ & $-0.002\,(5)$ & $-0.5\,(4)$ & $+0.7\,(5)$ & $0.000\,(6)$\\
\multicolumn{7}{@{}l}{\emph{Paper setting, $\alpha=0.95$}}\\
Teleport $-$ no PE & $+0.10\,(5)$ & $-0.009\,(8)$ & $-0.016\,(3)$ & $+1.2\,(7)$ & $+0.4\,(3)$ & $-0.005\,(8)$\\
Teleport $-$ original $K=8$ & $0.00\,(4)$ & $-0.004\,(5)$ & $-0.009\,(3)$ & $+0.1\,(7)$ & $+0.1\,(6)$ & $-0.003\,(6)$\\
Teleport $-$ shuffled & $+0.06\,(6)$ & $-0.006\,(8)$ & $+0.020\,(9)$ & $+1.6\,(9)$ & $+1.4\,(9)$ & $-0.004\,(9)$\\
Shuffled $-$ no PE & $+0.04\,(5)$ & $-0.003\,(7)$ & $-0.036\,(0)$ & $-0.4\,(4)$ & $-1.0\,(2)$ & $-0.001\,(6)$\\
DA + teleport $-$ DA & $-0.20\,(3)$ & $+0.001\,(6)$ & $+0.002\,(5)$ & $0.0\,(4)$ & $+0.5\,(6)$ & $+0.002\,(2)$\\
$\alpha=0.95$ $-$ $\alpha=0.85$ & $-0.16\,(2)$ & $+0.003\,(3)$ & $-0.001\,(5)$ & $-0.6\,(4)$ & $-0.3\,(3)$ & $+0.001\,(4)$\\
\midrule
\multicolumn{7}{@{}l}{\emph{Loops setting, $\alpha=0.85$}}\\
Teleport $-$ no PE & $+0.31\,(7)$ & $-0.004\,(6)$ & $-0.015\,(1)$ & $+0.8\,(7)$ & $-0.2\,(3)$ & $0.000\,(4)$\\
Teleport $-$ original $K=8$ & $+0.51\,(7)$ & $-0.002\,(7)$ & $-0.023\,(0)$ & $0.0\,(4)$ & $0.0\,(4)$ & $+0.003\,(3)$\\
Teleport $-$ shuffled & $-0.20\,(4)$ & $-0.003\,(7)$ & $-0.001\,(4)$ & $+0.8\,(6)$ & $+0.7\,(7)$ & $-0.003\,(7)$\\
Shuffled $-$ no PE & $+0.51\,(7)$ & $-0.002\,(6)$ & $-0.014\,(0)$ & $0.0\,(5)$ & $-0.9\,(3)$ & $+0.003\,(4)$\\
DA + teleport $-$ DA & $+0.12\,(4)$ & $-0.001\,(6)$ & $+0.004\,(7)$ & $-0.1\,(4)$ & $+0.3\,(7)$ & $-0.002\,(6)$\\
\multicolumn{7}{@{}l}{\emph{Loops setting, $\alpha=0.95$}}\\
Teleport $-$ no PE & $+0.04\,(4)$ & $0.000\,(4)$ & $-0.010\,(2)$ & $+1.1\,(7)$ & $+0.1\,(6)$ & $0.000\,(5)$\\
Teleport $-$ original $K=8$ & $+0.23\,(6)$ & $+0.002\,(4)$ & $-0.018\,(0)$ & $+0.3\,(5)$ & $+0.3\,(6)$ & $+0.003\,(2)$\\
Teleport $-$ shuffled & $-0.25\,(2)$ & $-0.005\,(5)$ & $+0.006\,(7)$ & $+1.7\,(9)$ & $+1.6\,(8)$ & $-0.003\,(6)$\\
Shuffled $-$ no PE & $+0.29\,(7)$ & $+0.004\,(4)$ & $-0.016\,(0)$ & $-0.6\,(4)$ & $-1.5\,(1)$ & $+0.003\,(5)$\\
DA + teleport $-$ DA & $-0.25\,(2)$ & $-0.003\,(8)$ & $+0.006\,(7)$ & $-0.3\,(2)$ & $0.0\,(5)$ & $+0.001\,(2)$\\
$\alpha=0.95$ $-$ $\alpha=0.85$ & $-0.27\,(3)$ & $+0.004\,(4)$ & $+0.005\,(5)$ & $+0.3\,(5)$ & $+0.3\,(7)$ & $0.000\,(4)$\\
\bottomrule
\end{tabular}
\caption{ProcRosetta, teleportation PNRWPE at $K=8$: mean paired differences over seeds 13 to 22, with the number of seeds in which the first arm is better given in parentheses (lower is better for behavior L1 and the net-to-tree loss). The exact-tree rate and top-1 retrieval (N$\to$T: net$\to$tree, T$\to$N: tree$\to$net) differ in percentage points. \emph{Shuffled}: teleportation PNRWPE permuted across all nodes of each net. \emph{DA}: DA-PNGNN with the linear gate. The final row of each setting isolates the effect of slower mixing, and finds no consistent effect.}
\label{tab:teleport-pr}
\end{table}

\end{document}